\documentclass[12pt]{article}

\usepackage[margin=2.5cm]{geometry}
\usepackage{times}
\usepackage[utf8]{inputenc}
\usepackage[english]{babel}

\usepackage{amsmath, amssymb, amsthm, mathtools, bm, bbm}
\allowdisplaybreaks

\usepackage{graphicx}
\usepackage{booktabs, multirow}
\usepackage{subcaption} 

\usepackage{algorithm}
\usepackage[noend]{algpseudocode}

\usepackage[numbers]{natbib}
\usepackage{hyperref}
\usepackage{authblk}
\usepackage{setspace}
\numberwithin{equation}{section}

\title{\bf Approximating Smooth Functionals with ReLU Networks}
\author{Shuhao Jiao\thanks{shuhao.jiao@cityu.edu.hk}}
\affil{Department of Biostatistics\\ City University of Hong Kong}
\date{}

\begin{document}
	\maketitle			
	\setlength\parindent{0pt}
	\setlength{\parskip}{1em}
	\theoremstyle{definition}
	\newtheorem{theorem}{Theorem}
	\newtheorem{assumption}{Assumption}
	\newtheorem{lemma}{Lemma}
	\newtheorem{remark}{Remark}
	\newtheorem{proposition}{Proposition}
	\newtheorem{definition}{Definition}
	\newtheorem{corollary}{Corollary}
	\newtheorem{example}{Example}

\begin{abstract}
We study the uniform approximation of smooth scalar-valued functionals on an infinite-dimensional separable Hilbert space by  ReLU neural networks. A key feature in deep learning for functional data is the varying importance of different coordinates/dimensions. Representing the functional input in a basis expansion, we quantify the importance of each coordinate through both the magnitude of its corresponding basis score and the directional sensitivity of the target functional. Our analysis combines coordinate truncation, anisotropic partitioning, local Taylor approximation, and ReLU network realization, while allowing unrestricted interactions among the retained coordinates. We establish a general nonasymptotic upper bound for the uniform approximation error and a complementary pseudo-dimension-based lower bound for the worst-case approximation error.  Under generalized exponential coordinate decay $w_ds_d\asymp\exp(-cd^\rho)$, with $\rho>0$, the upper and lower bounds match at the leading order, which is stretched-exponential in the logarithm of the network size budget, and thus yield the nearly optimal approximation rate. This is the first work to characterize neural network approximation error for infinite-dimensional functional inputs explicitly through the joint dimensional decay of coordinate magnitudes and directional sensitivities.
\end{abstract}

\section{Introduction}

The approximation of nonlinear functionals defined on infinite-dimensional spaces is a fundamental problem in functional data analysis. Neural network architectures for functional inputs have a  long history. \citet{rossi2005functional} extended multilayer perceptrons to functional inputs and established universal approximation and consistency results, while \citet{rossi2005representation} investigated basis-based representations of functional observations for multilayer perceptrons and radial-basis-function networks. More recent studies have developed network architectures that more explicitly preserve and exploit functional structure. \citet{yao2021deep} introduced adaptive basis layers that learn task-dependent finite-dimensional representations, and \citet{thind2023deep} incorporated functional input layers into densely connected neural networks. \citet{rao2023nonlinear} proposed functional neural networks with continuous hidden layers and basis representations, whereas \citet{ruegamer2024functional} extended semi-structured neural networks to functional covariates. From an approximation-theoretic perspective, \citet{song2023approximation} derived quantitative approximation rates for smooth functionals using ReLU networks, and \citet{shi2025nonlinear} developed a kernel-embedding-based functional neural network together with corresponding approximation guarantees. These developments demonstrate the growing potential of neural networks for learning nonlinear relationships from functional inputs.

Let \(\mathcal H=L^2[0,1]\) denote the space of square-integrable functions on \([0,1]\), and let \(f_0\colon\mathcal H\to\mathbb R\) be an unknown scalar-valued functional. Given a fixed set of orthonormal basis functions \(\{\nu_d(t)\colon d\geq1,\ t\in[0,1]\}\) of \(\mathcal H\), each \(X(t)\in\mathcal H\) admits the representation \(X(t)=\sum_{d\geq1}\xi_d\nu_d(t)\), where \(\xi_d=\langle X,\nu_d\rangle\coloneq \int_0^1X(t)v_d(t)dt\). Although the input \(X(t)\) is intrinsically infinite-dimensional, any implementable neural network can depend on only finitely many coordinates. Therefore, for some $D\in\mathbb Z_+$, we approximate \(f_0(X)\) by a ReLU neural network \(f_\phi(\xi_1,\ldots,\xi_D)\) based on the first \(D\) basis scores and study the resulting uniform approximation error over a prescribed functional input domain. Under a suitable choice of basis (e.g., functional principal components), the leading scores can provide a parsimonious representation of the dominant variation in the functional input. In contrast, pointwise measurements are typically highly correlated, depend on the observation grid, and do not yield a transparent characterization of coordinate importance.

A central feature of this problem is that the basis coordinates generally contribute unequally to the target functional. The contribution of the \(d\)-th coordinate is governed by the directional sensitivity of \(f_0\) along \(\nu_d(t)\) and the associated score magnitude. Coordinates with both large score magnitudes and strong directional sensitivities must be approximated accurately, whereas those with small score magnitudes or weak directional sensitivities can be truncated with limited loss. We consider a natural setting in which directional sensitivity decays with the coordinate index $d$. Our objective is to characterize how this coordinate-decay structure, together with the smoothness of \(f_0\) and the network width--depth budget, determines the approximation power of ReLU neural networks.

There is a substantial literature on neural network approximation theory.
For ReLU networks, \citet{yarotsky2017error,yarotsky2018optimal}
derived quantitative upper and lower bounds for smooth and continuous
functions, while \citet{lu2021deep} established nearly optimal width--depth
approximation rates using local Taylor expansions. Approximation guarantees
under Sobolev and Besov regularity have been developed by
\citet{guhring2020error} and \citet{siegel2023optimal}. Complementary results
have characterized the roles of sparse connectivity, parameter encoding, and
network depth \citep{bolcskei2019optimal,elbrachter2021deep,
yarotsky2020phase}. For structured target functions,
\citet{petersen2018optimal} obtained optimal rates for piecewise smooth
functions, \citet{shen2019nonlinear} studied approximation through
compositional structures, and \citet{shen2022deep} characterized approximation
rates in terms of network width and depth. We refer to
\citet{devore2021neural} for a comprehensive review of neural network
approximation theory. 

However, existing theory does not explicitly characterize how the decay rates in the input scores and the directional sensitivities of the target functional jointly affect network approximation. While \citet{song2023approximation} and \citet{shi2025nonlinear} study neural-network approximation of functional operators, they do not explicitly characterize how coordinate-wise decay of function scores and anisotropic directional sensitivities jointly affect approximation. Related coordinate-decay phenomena have been studied extensively in functional linear models; see, for example, \citet{cai2006prediction} and \citet{hall2007methodology}. Those studies, however, primarily concern linear-functional estimation and do not address the approximation of nonlinear functionals by ReLU networks.  
Under coordinate decay, higher-index coordinates typically have smaller score magnitudes, and the target functional may also be less sensitive to perturbations in these directions. 
Consequently, only finitely many coordinates need to be retained to achieve a prescribed approximation accuracy, while the discarded coordinates contribute a truncation tail. 

Learning from functional data poses a distinctive capacity-allocation problem, as any implementable neural predictor must compress an intrinsically infinite-dimensional input into a finite representation. Retaining too few coordinates causes information loss, whereas retaining too many requires a fixed network size budget to be distributed across an increasingly high-dimensional input. We formalize this representation--capacity trade-off for general scalar-valued functionals. Working directly with the basis coefficients of the functional input, we quantify the relevance of each coordinate through its magnitude and the directional Fr\'echet derivatives of the target functional. 
We establish non-asymptotic upper bounds and complementary lower bounds for ReLU networks, thereby characterizing the best attainable representation error under a finite computational budget. These results provide theoretically justified scaling rules for selecting the input truncation dimension jointly with network width and depth. We further quantify the diminishing importance of high-order coordinates through a generalized exponential decay profile, a flexible and natural family that explicitly links the decay parameter to the approximation error rate. 

The remainder of the paper is organized as follows. Section \ref{prelimi} introduces the problem formulation, network construction, and assumptions on coordinate-wise score decay and anisotropic directional sensitivities. Section \ref{appro} develops the approximation error analysis for the proposed FNN class. Section \ref{sim} presents the numerical experiments, and Section \ref{conclusion} concludes the paper. Technical proofs are provided in the appendix. 

\section{Preliminaries}
\label{prelimi}
\subsection{Basis Representation and Reparametrization}
Let \(\mathcal H\) be a separable Hilbert space, and assume that \(X(t)\in\mathcal H\) and \(f_0\colon\mathcal H\to\mathbb R\). Let \(\{\nu_d(t)\}_{d\ge 1}\) be a fixed orthonormal basis of \(\mathcal H\), and write
\(
X(t)=\sum_{d\ge 1}\xi_d \nu_d(t),
\)
where \(\bm x=(\xi_1,\xi_2,\ldots)\) denotes the corresponding score sequence. Define the coordinate map
\(
\mathcal T\colon\mathcal H\to \ell^2,
\
\mathcal T(X)=\bm{x}.
\)
Since \(\mathcal T\) is one-to-one, the functional \(f_0\) admits a sequence-space representation, and there exists a function \(g_0\colon\ell^2\to\mathbb R\) such that
\(
f_0=g_0\circ \mathcal T.
\)
Define \(P_D\colon\mathcal H\to \mathrm{span}\{\nu_1,\ldots,\nu_D\}\subset \mathcal H\) be the orthogonal projection operator, say 
\(
P_DX=\sum_{d=1}^D \xi_d \nu_d,
\)
where \(D\in\mathbb N_+\).
To evaluate the directional sensitivity of \(f_0\) along different basis directions, define
\begin{equation}
\label{w}
w_d\coloneq
\max_{1\le s\le m}
\left\{
\sup_{X\in\mathcal H}
\bigl|
\mathcal D^s f_0(X)[\nu_d,\ldots,\nu_d]
\bigr|
\right\}^{1/s},
\end{equation}
where \(\mathcal D^s f_0(X)\) denotes the \(s\)-th Fr\'echet derivative of \(f_0\) at \(X(t)\).
Clearly, large \(w_d\) indicates that \(f_0\) varies more rapidly along the \(d\)-th basis direction, while small \(w_d\) indicates weaker sensitivity.

\begin{remark}
The power \(1/s\) in the definition \eqref{w}
normalizes derivatives of different orders. To see this, consider a perturbation along the \(d\)-th basis direction, say \(X+h_d\nu_d\). The Taylor expansion of \(f_0\) gives terms of the form
\(
D^s f_0(X)[\nu_d,\ldots,\nu_d] h_d^s .
\)
By the definition of \(w_d\),
\[
\left|D^s f_0(X)[\nu_d,\ldots,\nu_d]\right| |h_d|^s
\le(w_d|h_d|)^s.
\]
Thus \(w_d\) is the basic sensitivity scale paid each time the \(d\)-th direction appears in a Taylor expansion component. If \(|h_d|\le s_d\), the effective contribution of the \(d\)-th coordinate is therefore measured by \(w_ds_d\). This explains why the joint decay of the directional derivative scale \(w_d\) and the score envelope \(s_d\) governs the approximation error.
\end{remark}
We now introduce the following assumption.

\begin{assumption}
\label{ass1}
Assume that \(f_0\colon\mathcal H\to\mathbb R\) is \(m\)-times
continuously Fr\'echet differentiable on \(\mathcal H\). 
\end{assumption}

For \(s=1,\ldots,m\) and \(X(t)\in\mathcal H\), let \(\mathcal D^s f_0(X)[h_1,\ldots,h_s]\) denote the \(s\)-th Fr\'echet derivative of \(f_0\) at \(X(t)\) applied to the directions \(h_1,\ldots,h_s\). Let \(\Omega\subset \ell^2\) denote the domain of the score vector \(\bm x=(\xi_1,\xi_2\ldots)\), and let \(\mathcal C^m(\Omega)\) denote the class of functions on \(\Omega\) with continuous partial derivatives up to order \(m\). To relate the Fr\'echet derivatives of \(f_0\) on \(\mathcal H\) to the coordinate derivatives of its score representation, we introduce the following proposition.

\begin{proposition}
\label{prop1}
Let \(\mathcal L\colon\mathcal H\to \ell^2\) be a bounded linear operator, and let \(g\colon\ell^2\to\mathbb R\) be \(m\)-times Fr\'echet differentiable on an open set containing \(\mathcal LX\) for any \(X(t)\in\mathcal H\). Define
\(
f=g\circ \mathcal L,
\)
then \(f\) is \(m\)-times Fr\'echet differentiable. Then for \(s=1,\ldots,m\), and \(h_1,\ldots,h_s\in\mathcal H\),
\begin{equation*}
\label{eq:frechet-chain-rule}
\mathcal D^s f(X)[h_1,\ldots,h_s]=\mathcal D^s g(LX)[Lh_1,\ldots,Lh_s].
\end{equation*}
In particular, for $D\in Z_+$, if \(\mathcal L=\mathcal T\circ P_D\), it follows that
\(
\mathcal LX=(\xi_1,\ldots,\xi_D,0,0,\ldots)^\top, 
\)
and
\begin{equation*}
\mathcal D^s f(X)[h_1,\ldots,h_s]
=
\sum_{d_1,\ldots,d_s=1}^D
\frac{\partial^s g}{\partial \xi_{d_1}\cdots\partial \xi_{d_s}}(\mathcal LX)\times
\prod_{j=1}^s \langle h_j,\nu_{d_j}\rangle.
\end{equation*}
Consequently, taking \(h_j=\nu_{d_j'}\) for \(j=1,\ldots,s\), where \(d_j'\in\{1,\ldots,D\}\), yields
\begin{equation*}
\label{eq:directional-partial-equivalence}
\mathcal D^s f(X)[\nu_{d_1'},\ldots,\nu_{d_s'}]
=\frac{\partial^s g}{\partial \xi_{d_1'}\cdots\partial \xi_{d_s'}}(\mathcal LX).
\end{equation*}
\end{proposition}
Proposition \ref{prop1} shows that directional Fr\'echet derivatives of \(f_0\) along the basis \(\{\nu_d\}\) coincide with coordinatewise partial derivatives of \(g_0\) in the corresponding score representation. 
This representation is the key device that allows us to transfer directional sensitivity of  functional into a form amenable to network approximation.

\subsection{Network Construction}
\label{s2.3}
A major practical challenge is that a functional input \(X(t)\in \mathcal H\) is infinite-dimensional, whereas any implementable neural network must take a finite-dimensional vector as input. To bridge this gap, we employ truncation. For a truncation level \(D\in\mathbb N_+\), we approximate \(X(t)\) by its projection
\(
P_DX(t)=\sum_{d=1}^D \xi_d \nu_d(t),
\)
and use the truncated score vector
\(
\bm{x}^{(D)}\coloneq(\xi_1,\ldots,\xi_D)
\)
as the network input. 
We consider functions
\(
f_\phi\colon\mathbb R^{D}\to\mathbb R,
\)
parameterized by $\phi$, belonging to a class $\mathcal F_{M,W,S}$ of feedforward neural networks with depth $M$, width $W$, and size $S$. 
We adopt the standard multilayer perceptron (MLP) architecture, where $f_\phi$ is expressed as a composition of affine transformations and nonlinear activation functions
\(
f_\phi(\bm{x}^{(D)})=\mathcal L_M \circ \rho \circ \mathcal L_{M-1}
\circ \cdots \circ \rho\circ \mathcal L_1 \circ \rho \circ \mathcal L_0(\bm{x}^{(D)}).
\)
where $\rho(x)=\max\{0,x\}$ is the ReLU activation function applied componentwise, and each affine map is of the form
\(
\mathcal L_i(x)=W_i x + b_i,\ i=0,1,\ldots,M.
\)
Here $W_i\in\mathbb R^{p_{i+1}\times p_i}$, $x\in\mathbb R^{p_i}$ and $b_i\in\mathbb R^{p_{i+1}}$, with $(p_0,p_1,\ldots,p_M,p_{M+1})$ denoting the widths of each layer, where $p_0=D$ and $p_{M+1}=1$. 

In the sequel, let $a_n\asymp b_n$ denote $C^{-1}b_n<a_n\le Cb_n$ for some $C>1$, $a_n\lesssim b_n$ denote $a_n\le Cb_n$ for some $C>0$, and $a_n\gtrsim b_n$ denote $a_n\ge Cb_n$ for some $C>0$, where ``$C$" denotes a generic positive constant.

\section{Error Analysis}
\label{appro}
\subsection{Coordinatewise decay}
\label{subsec:smoothness-coordinate-decay}

To obtain error bounds, we need to quantify how much each basis coordinate contributes to the  functional. This contribution is determined by two quantities: the size of the corresponding score and the sensitivity of the  functional along that basis direction. Coordinates with small scores
or weak directional sensitivity can be truncated with little loss, whereas coordinates with large values of both quantities must be retained by the network. The next assumption formalizes this joint coordinate-wise control through the score envelopes \(\{s_d\colon d\ge1\}\) and the directional sensitivity weights \(\{w_d\colon d\ge1\}\).

\begin{assumption}
\label{ass2}
There exists a constant \(C_0\ge 1\) and non-increasing sequence $\{w_d\}, \{s_d\}$ such that, for every finite multi-index \(\bm{\alpha}=(\alpha_1,\alpha_2,\ldots)\) with
\(1\le \|\bm{\alpha}\|_1\le m\), 
\(
\sup_{\bm{x}\in\Omega}
\left|
\partial^{\bm{\alpha}} g_0(x)
\right|
\le C_0\prod_{d\ge1} w_d^{\alpha_d}.
\)
In addition, 
\(|\xi_d|\le s_d,\ d\ge1.\)
\end{assumption}
As $w_d\to0$, the  functional becomes progressively less sensitive to higher-index basis directions. 
The next proposition shows that the directional-derivative decay condition in Assumption~\ref{ass2} is implied by product-type decay of higher-order interaction coefficients, thereby providing a concrete structural justification for this assumption. Assumption~\ref{ass2} holds for many regression function in statistical models, such as functional single-index model \citet{chen2011single}, functional additive model \citet{muller2008functional}, and functional quadratic regression model \citet{yao2010functional}.

\begin{proposition}
\label{prop:product-decay-interaction}
Let \(p\ge1\) be fixed and consider the function of the following form
\begin{align*}
f_0(X)&=G\left(Q_p(X)\right),
\qquad \xi_d=\langle X,\nu_d\rangle,\\
Q_p(X)&=
\sum_{1\le j_1\le\cdots\le j_p}
\eta_{j_1,\ldots,j_p}
\prod_{\ell=1}^p \xi_{j_\ell}.
\end{align*}
Assume that \(G\) has bounded derivatives up to order \(m\), that
\(\sum_{d\ge1}w_ds_d<\infty\), and that the interaction coefficients satisfy
the product-type decay condition
\begin{equation}
\label{prodecay}
|\eta_{j_1,\ldots,j_p}|
\lesssim
\prod_{\ell=1}^p w_{j_\ell},
\qquad
1\le j_1\le\cdots\le j_p .
\end{equation}
Then, for every \(1\le s\le m\) and every \(d_1,\ldots,d_s\ge1\),
\[
\left|
\mathcal D^s f_0(X)[\nu_{d_1},\ldots,\nu_{d_s}]
\right|
\lesssim
\prod_{\ell=1}^s w_{d_\ell}.
\]
\end{proposition}

The product-type condition \eqref{prodecay} can be viewed as a higher-order analogue of coefficient decay in functional linear regression (see e.g., \citet{cai2006prediction}). Indeed, when \(p=1\), the interaction index reduces to the linear form \(Q_1(X)=\sum_{j\ge1}\eta_j\xi_j\), and the condition becomes
\(|\eta_j|\lesssim w_j\). Thus \(w_j\) plays the role of the coordinate-wise coefficient scale. For \(p\ge2\), the same principle requires interactions involving weak directions to be small, that is, if any participating coordinate has a small sensitivity scale, then the corresponding higher-order interaction coefficient is also small. This is reasonable because an interaction term should not exert a pronounced effect
if it involves a basis direction whose individual contribution is already weak.

\subsection{Approximation framework}
\label{subsec:approximation-dimensional-decay}

The network construction builds on the local Taylor approximation framework (see, e.g., \citet{lu2021deep}), extending it through coordinate truncation and anisotropic partitioning to exploit coordinate-dependent decay. The basic idea is to partition the input domain into small cells, approximate the local Taylor expansion on each cell, fit the Taylor coefficient maps on the resulting discrete grid, and then approximate the monomial terms appearing in the Taylor expansion by product networks. 

For \(D\ge 1\), define the projected functional
\(
f_0^{(D)}\coloneq f_0\circ P_D.
\)
Since \(P_D\) is a bounded linear operator, Proposition~\ref{prop1}
implies that, for \(s=1,\ldots,m\),
\begin{align*}
\mathcal D^s f_0^{(D)}(X)[h_1,\ldots,h_s]=
\mathcal D^s f_0(P_DX)[P_Dh_1,\ldots,P_Dh_s].
\end{align*}
Consequently,
\(
\mathcal D^s f_0^{(D)}(X)
[\nu_{d_1},\ldots,\nu_{d_s}]
=0
\)
whenever at least one of \(d_1,\ldots,d_s\) is greater than \(D\). The Taylor expansion of the projected functional \(f_0^{(D)}\) at \(X_0(t)\) is
\begin{align*}
f_0^{(D)}(X)
&=
f_0^{(D)}(X_0)
+
\sum_{s=1}^{m-1}\frac{1}{s!}
\sum_{d_1,\ldots,d_s=1}^{D}
\mathcal D^sf_0(P_DX_0)
[\nu_{d_1},\ldots,\nu_{d_s}]
\prod_{j=1}^s\xi_{h,d_j}
+R_{m,D}(X,X_0),
\end{align*}
where
\begin{align*}
R_{m,D}(X,X_0)
&=\frac{1}{(m-1)!}
\int_0^1(1-\tau)^{m-1}
\mathcal D^mf_0(P_DX_0+\tau h_D)
[h_D,\ldots,h_D]\,d\tau.
\end{align*}
To approximate the original functional \(f_0\), we use the decomposition
\(
f_0(X)=f_0^{(D)}(X)+\Delta_D(X).
\)
Under Assumption~\ref{ass2}, the projection error satisfies
\begin{align*}
|\Delta_D(X)|
&=
\Bigg|
\int_0^1
\mathcal Df_0\bigl(P_DX+\tau(I-P_D)X\bigr)[(I-P_D)X]\,d\tau
\Bigg| \\
&\le
C_0\sum_{d>D}w_d|\xi_d|
\le
C_0\sum_{d>D}w_ds_d.
\end{align*}
Thus, the approximation of \(f_0\) consists of a finite-dimensional
network approximation error for \(f_0^{(D)}\) and a projection error
arising from the discarded coordinates.
The goal here is to approximate the above Taylor expansion up to order $m-1$ through the following steps:
\begin{itemize}
\item[1.]
For a non-increasing sequence \(\{L_d\colon d\ge1\}\), define \(D_0\coloneq\max\{d\colon L_d\ge1\}\), and let \(D\ge D_0\) denote the dimension of the input scores. Partition \(\Omega_{D_0}\coloneq[-s_1,s_1]\times\cdots\times[-s_{D_0},s_{D_0}]\) along its first \(D_0\) coordinates. Let \(\delta_d\in(0,2s_d/L_d)\), put \(a_{d,i}\coloneq-s_d+2s_di/L_d\), and construct the trifling region
\begin{align*}
\mathcal R(\Omega_{D_0},\bm L,\bm\delta)&=\bigcup_{d=1}^{D_0}\Biggl\{\bm\xi\in\Omega_{D_0}\colon\xi_d\in\bigcup_{i=1}^{L_d-1}\left(a_{d,i}-\delta_d,a_{d,i}\right]\Biggr\}.
\end{align*}
We denote the resulting cells in \(\Omega_{D_0}\setminus\mathcal R(\Omega_{D_0},\bm L,\bm\delta)\) by \(P_{\bm\theta}\), where \(\theta_d\in\{0,1,\ldots,L_d-1\}\). Then construct a neural network \(\tilde\psi=(\psi_1,\ldots,\psi_{D_0})\) that approximately maps each \(\bm x^{(D_0)}\in P_{\bm\theta}\) to
\begin{equation}
\label{cell}
\left(
-s_1+\frac{2s_1\theta_1}{L_1},
\ldots,
-s_{D_0}+\frac{2s_{D_0}\theta_{D_0}}{L_{D_0}}
\right).
\end{equation}
Thus, only the first \(D_0\) coordinates are partitioned, and the number of cells is \(\prod_{d=1}^{D_0}L_d\).

The decrease in \(L_d\) reflects the heterogeneous importance of the coordinate directions. The quantity \(L_d\) is the number of subintervals assigned to the \(d\)-th coordinate. Since the range of this coordinate has length \(2s_d\), its cell width is of order \(s_d/L_d\). Under the directional derivative bound, a perturbation of this magnitude changes the target functional by at most a constant multiple of \(w_ds_d/L_d\). Clearly, coordinates with larger \(w_d\) require finer partitions.

\item[2.]
For each cell \(P_{\bm\theta}^{(D_0)}\), define
\begin{align*}
P_{\bm\theta}^{(D)}
&\coloneq
P_{\bm\theta}^{(D_0)}
\times
\prod_{d=D_0+1}^{D}[-s_d,s_d],\\
X_{\bm\theta}(t)
&\coloneq
\sum_{d=1}^{D_0}
\left(
-s_d+\frac{2s_d\theta_d}{L_d}
\right)\nu_d(t).
\end{align*}
Thus, \(\bm x^{(D)}\in P_{\bm\theta}^{(D)}\) if and only if its first \(D_0\) coordinates \(\bm x^{(D_0)}=(\xi_1,\ldots,\xi_{D_0})\) belong to \(P_{\bm\theta}^{(D_0)}\); the remaining coordinates \((\xi_{D_0+1},\ldots,\xi_{D})\) do not determine the cell index \(\bm\theta\). For \(s=1,\ldots,m-1\) and \(\bm d=(d_1,\ldots,d_s)\in\{1,\ldots,D\}^s\), define the piecewise-constant Taylor coefficient map by
\begin{align*}
\mathcal Q_{\bm d}^{(s)}(\bm x^{(D)})
&\coloneq
\sum_{\bm\theta}
\mathcal D^s f_0^{(D)}(X_{\bm\theta})[\nu_{d_1},\ldots,\nu_{d_s}]
\mathbf 1_{P_{\bm\theta}^{(D_0)}}(\bm x^{(D_0)}).
\end{align*}
for \(\bm x^{(D)}\in
\bigcup_{\bm\theta}P_{\bm\theta}^{(D)}\).
Since \(X_{\bm\theta}(t)\in\operatorname{span}
\{\nu_1,\ldots,\nu_{D_0}\}\) and \(d_1,\ldots,d_s\le D\),
\[
\mathcal D^s f_0^{(D)}(X_{\bm\theta})
[\nu_{d_1},\ldots,\nu_{d_s}]=
\mathcal D^s f_0(X_{\bm\theta})
[\nu_{d_1},\ldots,\nu_{d_s}].
\]
In particular, if
\(\bm x^{(D_0)}\in P_{\bm\theta}^{(D_0)}\), then
\begin{align*}
\mathcal Q_{\bm d}^{(s)}
\left(
\xi_1,\ldots,\xi_{D_0},\xi_{D_0+1},\ldots,\xi_{D}
\right)&=
\mathcal Q_{\bm d}^{(s)}
(
\xi_1,\ldots,\xi_{D_0},
0,\ldots,0
)\\
&=
\mathcal D^s f_0(X_{\bm\theta})
[\nu_{d_1},\ldots,\nu_{d_s}].
\end{align*}
We then construct a network \(\phi_{\bm \alpha}(\bm x^{(D)})\) to approximate \(\mathcal Q_{\bm d}^{(s)}(\bm x^{(D)})\) through point fitting, where $\bm\alpha$ indicates the multi-index and $\|\bm\alpha\|_1=s$. For each fixed \(\bm d\), the coefficient map has only \(\prod_{d=1}^{D_0}L_d\) distinct values, one for each \(D_0\)-dimensional cell \eqref{cell}. 

\item[3.]
For \(\bm x^{(D_0)}\in P_{\bm\theta}\), write
\[
\xi_{h_{\bm\theta},d}
\coloneq
\begin{cases}
\displaystyle
\xi_d+s_d-\frac{2s_d\theta_d}{L_d},
& d\le D_0,\\[6pt]
\xi_d,
& D_0<d\le D.
\end{cases}
\]
For
\(\bm d=(d_1,\ldots,d_s)\in\{1,\ldots,D\}^s\), construct a network
\(P_{\bm \alpha}(\cdot)\) that approximates the monomial
\[
\prod_{j=1}^s\xi_{h_{\bm\theta},d_j}=
\prod_{d=1}^{D_0}
\left\{
\xi_d-\psi_d(\bm x^{(D_0)})
\right\}^{\alpha_d}
\prod_{d=D_0+1}^{D}\xi_d^{\alpha_d}.
\]
In particular, we first construct a network \(\phi_{\times}\) that
approximates the scalar product map \((x,y)\mapsto xy\), and then
compose copies of \(\phi_{\times}\) to approximate the required
monomials.
\end{itemize}



Then the neural network construction can be written as
\begin{align*}
\phi(\bm x^{(D)})
&=
\sum_{\|\bm\alpha\|_1\le m-1}
\phi_{\times}
\Biggl(
\frac{\phi_{\bm\alpha}(\bm x^{(D)})}{\bm\alpha!},
P_{\bm\alpha}
\left(
\bm x^{(D_0)}-\tilde\psi(\bm x^{(D_0)}),\xi_{D_0+1},\ldots,\xi_{D}
\right)
\Biggr),
\end{align*}
where \(\bm\alpha=(\alpha_1,\ldots,\alpha_{D})\),
\(\bm\alpha!\coloneq\prod_{d=1}^{D}\alpha_d!\). 

Fix integers \(D\) define the finite-dimensional slice \(\Omega_{D}\coloneq\prod_{d=1}^{D}[-s_d,s_d]\). Let \(\bm w=(w_1,\ldots,w_{D})\), and let \(\mathcal C^m(\Omega_{D},\bm w)\) be the class of functions with continuous partial derivatives up to order \(m\) satisfying \[\|\partial^{\bm\alpha}g\|_{L^\infty(\Omega_{D})}\le C_0\prod_{d=1}^{D}w_d^{\alpha_d}\] for every multi-index \(\bm\alpha\in\mathbb N_0^{D}\) with \(1\le\|\bm\alpha\|_1\le m\). 
Let
\[
H(D)\coloneq\sum_{d>D}w_ds_d,
\qquad
\Psi(D)\coloneq\sum_{d=1}^D\log\{(w_ds_d)^{-1}\},
\]
then we introduce the following results on approximation error.
\begin{theorem}
\label{upper}
Suppose that Assumption \ref{ass1} and \ref{ass2} hold. 
For $W_\phi,M_\phi\in\mathbb Z_+$, let
$\mathcal A=W_\phi M_\phi$ and
\begin{align*}
D_0 = \max_d\left\{
w_{d}s_{d}\exp\left[
\frac{\log\mathcal A^2+\Psi(d)}{d}
\right]\ge2
\right\}.
\end{align*}
Suppose that \(D\) satisfies \(H(D)\lesssim H(D_0)^m\). Then there
exists a ReLU network class \(\mathcal F\) with input dimension \(D\),
size \(S\), width \(W\), and depth \(M\) satisfying
\(
S\lesssim 3^{D_0}\{D^{m-1}W_\phi^2M_\phi+1\},
\)
\(
W\lesssim 3^{D_0}\{D^{m-1}W_\phi+1\},
\)
and \(M\lesssim M_\phi+2D_0\), such that
\begin{align*}
\sup_{f_0\in\mathcal C^m(\Omega,\bm w)}
\inf_{\phi\in\mathcal F}
\|\phi-f_0\|_{L^\infty(\Omega)}\lesssim
D_0^m
\exp\left[
-\frac{m}{D_0}
\{\log\mathcal A^2+\Psi(D_0)\}
\right].
\end{align*}
\end{theorem}

Theorem \ref{upper} gives the approximation bound
\[
        D_0^m
        \exp\left[
        -\frac{m}{D_0}\{\log\mathcal A^2+\Psi(D_0)\}
        \right].
\]
We refer to \(D_0\) as the {\it effective dimension}, as it quantifies the number of input coordinates that can be retained in the network approximation under a given network size budget.
Ignoring the polynomial prefactor \(D_0^m\), the leading exponential
term in Theorem~\ref{upper} can be written as
\(\exp\{-m\varphi(\log\mathcal A^2)\}\), where
\(\varphi(t)=\{t+\Psi(D_0(\mathcal A))\}/D_0(\mathcal A)\).
To determine whether the resulting approximation rate can be
polynomial in the network size budget, it is useful to compare
\(\varphi(t)\) with \(t=\log\mathcal A^2\). By definition,
\(\varphi(t)/t
=1/D_0(\mathcal A)
+\Psi(D_0(\mathcal A))/\{tD_0(\mathcal A)\}\).
Thus, when the effective dimension \(D_0(\mathcal A)\) diverges and
the cumulative coordinate-decay term \(\Psi(D_0(\mathcal A))\)
grows more slowly than \(tD_0(\mathcal A)\), the exponent satisfies
\(\varphi(t)=o(t)\) and then the approximation error decay in a subpolynomial rate. 
Intuitively, as more coordinates must be
retained, the available network size budget is distributed over an expanding input space, preventing the leading exponent from being of order \(t\) and the polynomial rate cannot be achieved. 


The following proposition gives the lower bound of the approximation error. 

\begin{theorem}
\label{lower}
Let \(J\geq1\), and let \(\mathcal F\) be any approximation class with pseudo-dimension \(V=\operatorname{PDim}(\mathcal F)\). Suppose that \(\mathcal F\) uniformly approximates \(\mathcal C^m(\Omega,\bm w)\) with error \(\varepsilon\), that is,
\[
\sup_{f\in\mathcal C^m(\Omega,\bm w)}\inf_{\phi\in\mathcal F}\|\phi-f\|_{L^\infty(\Omega)}=\varepsilon.
\]
For any $J$ such that \(\left(\frac{C_0}{2\varepsilon}\right)^{1/m}w_Js_J\geq1,\) we have
\[
\varepsilon\gtrsim\exp\left[-\frac{m}{J}\{\log V+\Psi(J)\}\right].
\]
\end{theorem}

Under generalized exponential coordinate decay \(w_ds_d\asymp\exp(-cd^\rho)\), applying Theorem~\ref{lower} with \(J=D_0\), together with the pseudo-dimension bound \(\operatorname{PDim}(\mathcal F)\lesssim SM\log S\) \citep{bartlett2019nearly}, shows that the lower bound has the same leading exponential term as Theorem~\ref{upper}. The discrepancy consists only of lower-order multiplicative factors, which is negligible relative to the common leading exponent. Consequently, the constructive upper bound is nearly optimal on the leading exponential scale. The detailed comparison of the two bounds is provided below.

To compare the lower bound with the constructive upper bound, apply
Theorem~\ref{lower} with \(J=D_0\) to the restriction of the constructed
network class to the \(D_0\)-dimensional slice obtained by setting
\(\xi_{D_0+1}=\cdots=\xi_D=0\). Restricting the input domain cannot
increase the pseudo-dimension. Therefore, Theorem~\ref{lower} gives
\(\varepsilon_{\mathrm{low}}\gtrsim
\exp[-m\{\log V+\Psi(D_0)\}/D_0]\), provided that
\(\{C_0/(2\varepsilon)\}^{1/m}w_{D_0}s_{D_0}\geq1\). If this
admissibility condition fails, then
\(\varepsilon>(C_0/2)(w_{D_0}s_{D_0})^m\). By the definition of
\(D_0\), we have \(\varepsilon\gtrsim(w_{D_0}s_{D_0})^m\gtrsim
\exp[-m\{\log\mathcal A^2+\Psi(D_0)\}/D_0]\). Thus, failure of the
admissibility condition also implies a lower bound with the same leading
exponential rate.

For the constructed network class,
\(V=\operatorname{PDim}(\mathcal F)\lesssim SM\log S\). Since
\(S\lesssim3^{D_0}\{D^{m-1}W_\phi^2M_\phi+1\}\),
\(M\lesssim M_\phi+2D_0\), and
\(\mathcal A=W_\phi M_\phi\), it follows that
\(V\lesssim3^{D_0}D^m\mathcal A^2
\log(3^{D_0}D^m\mathcal A^2)\). Consequently,
\begin{align*}
\varepsilon_{\mathrm{low}}
\gtrsim{}&
\exp[-m\{\log\mathcal A^2+\Psi(D_0)\}/D_0]\times3^{-m}D^{-m^2/D_0}
\{\log(3^{D_0}D^m\mathcal A^2)\}^{-m/D_0}.
\end{align*}
On the other hand, Theorem~\ref{upper} gives
\(\varepsilon_{\mathrm{up}}\lesssim
D_0^m\exp[-m\{\log\mathcal A^2+\Psi(D_0)\}/D_0]\). Thus, the upper
and lower bounds share the same leading exponential term, while their
discrepancy is confined to the multiplicative factor
\(3^mD_0^mD^{m^2/D_0}
\{\log(3^{D_0}D^m\mathcal A^2)\}^{m/D_0}\).

Under the generalized exponential coordinate decay condition
\(w_ds_d\asymp\exp(-cd^\rho)\), for each fixed \(\rho>0\), the tail satisfies
\(
\log H(j)=-cj^\rho+o(j^\rho).
\)
Fix any \(\kappa>m^{1/\rho}\), and choose an integer \(D\ge\kappa D_0\) such that \(D\asymp D_0\). Then, as \(\mathcal A\to\infty\),
\[
\log\frac{H(D)}{H(D_0)^m}
\le
-c(\kappa^\rho-m)D_0^\rho+o(D_0^\rho)
\longrightarrow-\infty.
\]
Thus, \(H(D)\lesssim H(D_0)^m.\)
Moreover, the definition of \(D_0\) gives
\(
D_0\asymp(\log\mathcal A^2)^{1/(\rho+1)}.
\)
It follows that
\(
D^{m^2/D_0}
=
\exp\{m^2\log(D)/D_0\}
=
1+o(1).
\)
Since
\(
\log(3^{D_0}D^m\mathcal A^2)
=
O(\log\mathcal A^2),
\)
we also have
\(
\{\log(3^{D_0}D^m\mathcal A^2)\}^{m/D_0}
=
1+o(1).
\)
Consequently, the multiplicative discrepancy is bounded by
\(
O(D_0^m)
=
O\!\left((\log\mathcal A^2)^{m/(\rho+1)}\right).
\)
Its logarithm is \(O(\log\log\mathcal A)\), which is negligible relative to the common leading exponent.
Thus, the constructive upper bound is nearly optimal on the leading exponential scale.


\subsection{Approximation Error Analysis under Generalized Exponential Decay}

Following the discussion in the previous section, we consider a generalized exponential coordinate decay regime to quantify the coordinate decay rate. We allow the $w_ds_d$ to decay as \(a^{-qd^\rho}\), where \(q\coloneq\tau+\tau_\omega\) and \(\rho>0\). This formulation contains three regimes in a unified way: \(0<\rho<1\) corresponds to stretched-exponential decay, \(\rho=1\) recovers the ordinary exponential case, and \(\rho>1\) corresponds to super-exponential decay. This general exponential-decay regime is particularly natural for smooth functional inputs possessing rapidly decaying basis coefficients. 

Let \(q=\tau+\tau_\omega\), \(c_a=q\log a\),
\(\mathcal A^2=W_\phi^2M_\phi^2\), \(c'=mc_a^{1/(\rho+1)}\{(\rho+1)/\rho\}^{\rho/(\rho+1)}\), and
\(t=\log\mathcal A^2\). The following theorem gives a stretched-exponential approximation bound. 

\begin{theorem}
\label{thm:exp-approx-rate}

Suppose that Assumptions~\ref{ass1} and~\ref{ass2} hold. Fix
\(a>1\), \(\tau>0\), and \(\tau_\omega>0\). For \(\rho>0\), assume that 
\(s_d\asymp a^{-\tau d^\rho}\) and
\(w_d\asymp a^{-\tau_\omega d^\rho}\). 
Suppose that \(t\geq\log2\), and let \(D_0\) be defined as in Theorem~\ref{upper}, and choose \(D\geq D_0\)
such that \(H(D)\lesssim H(D_0)^m\). Then the ReLU network class
\(\mathcal F\) constructed in Theorem~\ref{upper} satisfies
\begin{equation}
\label{eq:exp-uniform-explicit-rate}
\inf_{f\in\mathcal F}\|f-f_0\|_{L^\infty}
\lesssim
\exp\{-c't^{\frac{\rho}{\rho+1}}(1+\delta_{\mathcal A,\rho})\},
\end{equation}
where
\begin{equation}
\label{remainder}
1+\delta_{\mathcal A,\rho}
=
\frac{D_0^{-1}\{t+\Psi(D_0)\}-\log D_0-\Psi(1)}
{c_a^{1/(\rho+1)}
\{t(\rho+1)/\rho\}^{\rho/(\rho+1)}}.
\end{equation}
and 
\(\delta_{\mathcal A,\rho}\to0\) in either of the following two
regimes: \(\mathcal A\to\infty\) with \(\rho>0\) fixed, or
\(\rho\to\infty\) with \(\mathcal A\geq2\) fixed.
\end{theorem}

Under this generalized exponential decay condition, it can be shown that 
\(D_0(\mathcal A)\asymp t^{1/(\rho+1)}\to\infty\) and
\(\Psi(D_0(\mathcal A))
\asymp D_0(\mathcal A)^{\rho+1}\asymp t\).
Consequently,
\(\varphi(t)\asymp t/D_0(\mathcal A)
\asymp t^{\rho/(\rho+1)}=o(t)\).
Therefore, up to polynomial factors in \(D_0(\mathcal A)\), the leading approximation rate is sub-polynomial in \(\mathcal A^2\). In particular, for every \(\alpha>0\), \(\exp\{-c\varphi(\log\mathcal A^2)\}\gg(\mathcal A^2)^{-\alpha}\) as $\mathcal A^2$ is sufficiently large. Together with the preceding lower-bound comparison, this shows that the best attainable approximation rate is sub-polynomial in the network size budget.

We now examine the bound as \(\mathcal A\to\infty\) for fixed \(\rho>0\) and as \(\rho\to\infty\) for fixed \(\mathcal A\). For every
fixed finite \(\rho>0\), the approximation error decays
sub-polynomially in the network size budget as
\(\mathcal A\to\infty\), reflecting the intrinsic difficulty of
approximating an unrestricted smooth functional with genuinely
infinite-dimensional inputs. By contrast, as \(\rho\to\infty\), the
coordinate importance decays increasingly rapidly, and the
approximation problem progressively approaches a univariate one, that is, the
upper-bound rate consequently converges to
\(\mathcal A^{-2m}\), the approximation rate for a \(\mathcal C^m\)-smooth univariate function (see \citet{lu2021deep}).

\section{Numerical experiment results}
\label{sim}

Let \(\{\nu_d(t)\colon d\geq1\}\) be the Fourier basis on \([0,1]\), and
generate the functional input by
\(X(t)=\sum_{d=1}^{D_{\mathrm{total}}}\xi_d\nu_d(t)\), where
\(D_{\mathrm{total}}=100\). Write \(\xi_d=s_{d,\rho}u_d\), with
\(u_d\) independently generated from \(\operatorname{Unif}[-1,1]\).
Define \(a_{d,\rho}=\exp(-\lambda d^\rho)\), where \(\lambda=0.5\), and
set \(s_{d,\rho}=w_{d,\rho}=a_{d,\rho}^{1/2}\), so that
\(w_{d,\rho}s_{d,\rho}=a_{d,\rho}\). We consider
\(\rho\in\{0.5,1,2\}\), corresponding respectively to stretched-
exponential, ordinary exponential, and super-exponential decay.

Let \(z_{d,\rho}=w_{d,\rho}\xi_d=a_{d,\rho}u_d\), and let
\(\mathcal I=\{(j_r,\ell_r)\colon r=1,\ldots,R\}\) be a fixed collection of
\(R=64\) distinct coordinate pairs sampled uniformly without replacement
from \(\{(j,\ell)\colon 1\leq j<\ell\leq D_{\mathrm{total}}\}\). We use the target
\begin{align*}
f_{0,\rho}(X)
={}&\sum_{k=1}^{K}b_k
\sin\Bigg\{
\beta_k+
\sum_{d=1}^{D_{\mathrm{total}}}\gamma_{kd}z_{d,\rho}+\frac{\eta}{\sqrt R}
\sum_{r=1}^{R}\delta_{kr}
\sin(z_{j_r,\rho})\sin(z_{\ell_r,\rho})
\Bigg\}.
\end{align*}
Here, \(K=128\) and \(\eta=0.75\). The parameters \(\beta_k\) are
independently generated from \(\operatorname{Unif}[0,2\pi]\), while
\(\gamma_{kd}\) and \(\delta_{kr}\) are independent Rademacher variables.
The coefficients \(b_k\) are independently generated from \(N(0,1)\) and
normalized so that \(\sum_{k=1}^K|b_k|=1\). The same realizations of
\(\mathcal I\), \(\beta_k\), \(\gamma_{kd}\), \(\delta_{kr}\), and \(b_k\)
are used for all values of \(\rho\), \(D\), and the network width.

For each \(D\in\{1,5,10,30,50,70\}\), define
\(f_{0,\rho}^{(D)}(X)=f_{0,\rho}(P_DX)\), where
\(P_DX=\sum_{d=1}^{D}\xi_d\nu_d\). A fully connected ReLU network is
trained to approximate \(f_{0,\rho}^{(D)}\) using
\((u_1,\ldots,u_D)\) as input. We use four hidden layers with common
width \(W\in\{2,4,8,16,32,64\}\), holding the depth and all other
architectural specifications fixed across \(D\) and \(\rho\). We generate
\(50{,}000\) training observations, \(10{,}000\) validation observations,
and \(50{,}000\) independent test observations. 
Optimization uses Adam with initial learning rate \(10^{-3}\), batch size \(1024\), and at most \(2000\) steps; early stopping is based on validation error. For each \((\rho,D,W)\), we train from ten independent initializations and retain the one with the smallest validation total error. For a width budget \(W\), all fitted networks with hidden-layer width no larger than \(W\) are admissible, and the final model is again selected using validation error only. We report the empirical total \(L^\infty\) error against the full target \(f_{0,\rho}(X)\) on the independent test sample.
Figure~\ref{fig:nonlinear-total-error-vs-width} shows that the total error is nearly insensitive to width when \(D=1\), because truncation dominates. Once more coordinates are retained, increasing width generally reduces the error before the improvement levels off. Faster coordinate decay produces smaller errors throughout, supporting the conclusion that it reduces both the effective dimension and the capacity required for accurate approximation.

\begin{figure*}[t!]
    \centering
    \includegraphics[width=\textwidth]{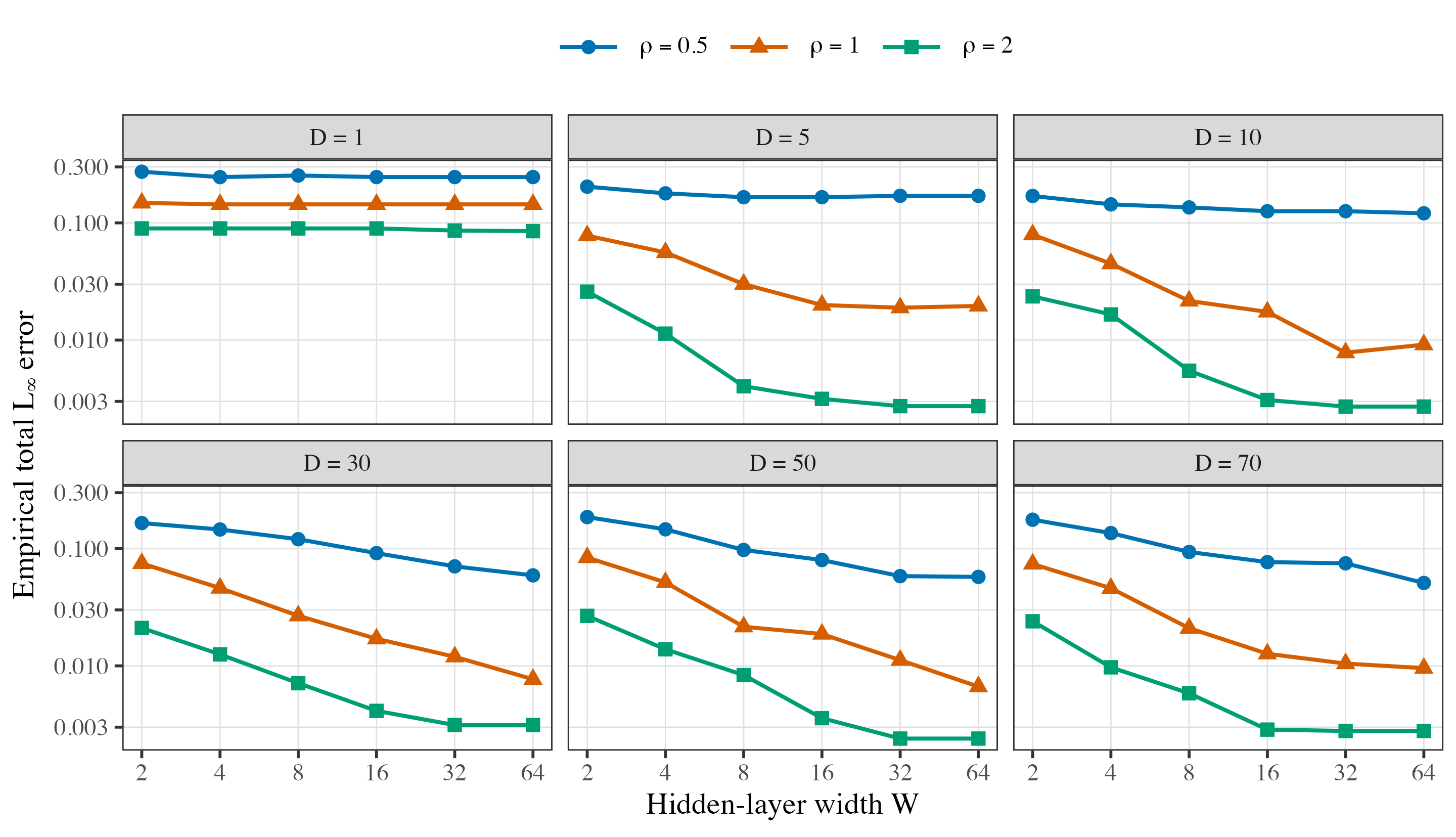}
    \caption{Total approximation error as a function of the common hidden-layer width.}
    \label{fig:nonlinear-total-error-vs-width}
\end{figure*}


\section{Conclusion}
\label{conclusion}

This paper studies the approximation of smooth scalar-valued functionals by ReLU networks under coordinate-dependent decay. Working with a fixed basis representation, we quantify the importance of each coordinate through its score magnitude and the directional  sensitivity of the target functional. By combining coordinate truncation, anisotropic partitioning, local Taylor approximation, and explicit ReLU network constructions, we derive a non-asymptotic uniform approximation bound. 
Under generalized exponential coordinate decay, the complementary lower bound matches the constructive upper bound on the leading exponential scale, establishing a nearly optimal approximation rate with stretched-exponential decay in the logarithm of the network size budget. 

This characterization shows how the dimension, which can be effectively retained by networks, grows with the network size and why, over the general functional class considered here, allowing unrestricted coordinate structures precludes a polynomial approximation rate. The slow decay of the approximation error suggests a potential limitation for network-based statistical estimation: reducing approximation bias may require a rapidly growing network, which in turn increases the network complexity that must be controlled during estimation. Establishing the resulting statistical rates, however, requires further analysis. Identifying structural restrictions that permit polynomial estimation rates and developing the corresponding theory remain directions for future work.


\bibliographystyle{agsm}
\bibliography{AFNN-arXiv}

\appendix

\setcounter{equation}{0}
\renewcommand{\theequation}{A.\arabic{equation}}
\renewcommand{\theHequation}{appendix.A.\arabic{equation}}

\section*{Appendix}


\begin{proof}[Proof of Proposition 1]
Note that since $L$ is bounded linear, it is Fr\'echet differentiable
everywhere with $\mathcal DL(X)=L$ and $\mathcal D^sL(X)=0$ for all $s\ge2$.
We prove by induction on $m$.
For $m=1$, the chain rule yields
\(
\mathcal Df(X)=\mathcal Dg(LX)\circ \mathcal DL(X)=\mathcal Dg(LX)\circ L,
\)
hence $\mathcal Df(X)[h]=\mathcal Dg(LX)[Lh]$, proving the result for $m=1$.

Assume the result holds for $s\le m-1$. Fix $h_1,\ldots,h_{m-1}\in\mathcal H$
and define \[G(u)\coloneq\mathcal D^{m-1}g(u)
[Lh_1,\ldots,Lh_{m-1}]\] for \(u\in\Omega\). Since $g$ is $m$-times
Fr\'echet differentiable, $G$ is Fr\'echet differentiable and
\(\mathcal DG(u)[v]=\mathcal D^mg(u)[v,Lh_1,\ldots,Lh_{m-1}]\) for
\(v\in\Omega\).
Using the definition of $\mathcal D^mf$ and the induction hypothesis,
\begin{align*}
&\mathcal D^{m}f(X)[h_1,\ldots,h_{m-1},h_m] \\
&\quad=
\lim_{t\to0}\frac{1}{t}
\Big\{\mathcal D^{m-1}f(X+t h_m)
[h_1,\ldots,h_{m-1}] 
-\mathcal D^{m-1}f(X)
[h_1,\ldots,h_{m-1}]\Big\} \\
&\quad=
\lim_{t\to0}\frac{1}{t}
\Big\{G\big(L(X+t h_m)\big)-G(LX)\Big\}.
\end{align*}
Since $L$ is linear, $L(X+t h_{m})=LX+tLh_{m}$, and the last limit equals
\(\mathcal DG(LX)[Lh_m]=\mathcal D^mg(LX)
[Lh_m,Lh_1,\ldots,Lh_{m-1}]\).
By symmetry of $\mathcal D^m g(LX)$, we may permute the arguments, proving the result for $m$. This completes the induction.
\end{proof}

\begin{proof}[Proof of Proposition 2]
We first prove the derivative bound for \(Q_p\). Since
\(\xi_j=\langle X,\nu_j\rangle\), we have
\(D\xi_j[\nu_d]=\langle \nu_d,\nu_j\rangle=\mathbf 1\{j=d\}\).
For a fixed monomial \(\eta_{j_1,\ldots,j_p}\prod_{\ell=1}^p\xi_{j_\ell}\),
taking \(q\) directional derivatives along
\(\nu_{d_1},\ldots,\nu_{d_q}\) amounts to differentiating \(q\) of the \(p\) scores. Therefore, each nonzero term after differentiation contains the coefficient \(\eta_{j_1,\ldots,j_p}\), \(q\) Kronecker factors that match \(q\) indices with \(d_1,\ldots,d_q\), and \(p-q\) remaining score factors. Hence,
\begin{align*}
D^q Q_p(X)[\nu_{d_1},\ldots,\nu_{d_q}]=\sum_{k_1,\ldots,k_{p-q}\ge1}
\eta^*_{d_1,\ldots,d_q,k_1,\ldots,k_{p-q}}
\xi_{k_1}\cdots \xi_{k_{p-q}},
\end{align*}
where \(\eta^*\) denotes the coefficient after arranging the indices in
nondecreasing order. By the product-type coefficient decay condition,
\(\left|\eta^*_{d_1,\ldots,d_q,k_1,\ldots,k_{p-q}}\right|
\lesssim(\prod_{a=1}^qw_{d_a})(\prod_{b=1}^{p-q}w_{k_b})\).
Using \(|\xi_k|\le s_k\), we obtain
\begin{align*}
\left|
D^q Q_p(X)[\nu_{d_1},\ldots,\nu_{d_q}]
\right|
&\lesssim\sum_{k_1,\ldots,k_{p-q}\ge1}\left(\prod_{a=1}^q w_{d_a}\right)
\left(\prod_{b=1}^{p-q} w_{k_b}s_{k_b}\right) \\
&=\left(\prod_{a=1}^q w_{d_a}\right)
\left(\sum_{k\ge1}w_ks_k\right)^{p-q}.
\end{align*}
Since \(\sum_{k\ge1}w_ks_k<\infty\), this gives
\(\left|D^qQ_p(X)[\nu_{d_1},\ldots,\nu_{d_q}]\right|
\lesssim\prod_{a=1}^qw_{d_a}\).
It remains to pass from \(Q_p\) to \(f_0=G\circ Q_p\). By the  chain rule, for \(1\le s\le m\),
\begin{align*}
D^sf_0(X)[\nu_{d_1},\ldots,\nu_{d_s}]=
\sum_{\pi\in\mathcal P_s}
G^{(|\pi|)}(Q_p(X))
\prod_{B\in\pi}
D^{|B|}Q_p(X)[\nu_{d_i}\colon i\in B],
\end{align*}
where \(\mathcal P_s\) denotes the set of all partitions of
\(\{1,\ldots,s\}\). Since \(Q_p\) is a polynomial of order \(p\),
\(D^qQ_p(X)=0\) for \(q>p\). Hence only blocks \(B\) with
\(|B|\le p\) contribute.
For each such block, the first part of the proof gives
\(\left|D^{|B|}Q_p(X)[\nu_{d_i}\colon i\in B]\right|
\lesssim\prod_{i\in B}w_{d_i}\).
Since the derivatives of \(G\) up to order \(m\) are bounded, each product in the partition expansion is bounded by
\(
\prod_{\ell=1}^s w_{d_\ell}
\)
up to some constant. The number of partitions depends only on \(s\le m\), so summing over
\(\pi\in\mathcal P_s\) yields
\(\left|D^sf_0(X)[\nu_{d_1},\ldots,\nu_{d_s}]\right|
\lesssim\prod_{\ell=1}^sw_{d_\ell}\).
This proves the proposition.
\end{proof}

For \(D\geq D_0\), define the finite-dimensional trifling region by
\[
\mathcal R_D
\coloneq
\mathcal R(\Omega_{D_0},\bm L,\bm\delta)
\times\prod_{d=D_0+1}^{D}[-s_d,s_d]
\subseteq\Omega_D,
\]
where
\begin{align*}
\mathcal R(\Omega_{D_0},\bm L,\bm\delta)
&=\bigcup_{d=1}^{D_0}
\left\{\bm x^{(D_0)}\in\Omega_{D_0}\colon
\xi_d\in\mathcal I_d\right\},\\
\mathcal I_d
&=\bigcup_{i=1}^{L_d-1}
\left(-s_d+\frac{2s_di}{L_d}-\delta_d,
-s_d+\frac{2s_di}{L_d}\right].
\end{align*}
For a continuous function \(g\) on \(\Omega_D\), define its
directional modulus of continuity by
\(
\omega_{g,d}(r)
\coloneq
\sup\{|g(\bm x+t\bm e_d)-g(\bm x)|\colon
\bm x,\bm x+t\bm e_d\in\Omega_D,\ |t|\leq r\}.
\)

\begin{proof}[Proof of Theorem \ref{upper}]
For \(X\in\mathcal H\) with score vector
\(\bm x=(\xi_1,\xi_2,\ldots)\in\Omega\), let
\(X_0=\sum_{d=1}^{D_0}\xi_{0,d}\nu_d\). Define
\(h_D\coloneq P_DX-X_0=\sum_{d=1}^D\xi_{h,d}\nu_d\),
where \(\xi_{h,d}=\xi_d-\xi_{0,d}\) for \(d\leq D_0\) and
\(\xi_{h,d}=\xi_d\) for \(D_0<d\leq D\).
Since \(P_{D}X_0=X_0\), the Taylor expansion of
\(f_0^{(D)}\) at \(X_0\) is
\begin{equation}
\label{taylor}
\begin{aligned}
f_0^{(D)}(X)
=f_0(X_0)&+\sum_{s=1}^{m-1}\frac{1}{s!}
\sum_{d_1,\ldots,d_s=1}^{D}
\mathcal D^sf_0(X_0)[\nu_{d_1},\ldots,\nu_{d_s}]\\
&\times\prod_{j=1}^s\xi_{h,d_j}
+R_{m,D}(X,X_0),
\end{aligned}
\end{equation}
where
\begin{align*}
R_{m,D}(X,X_0)
={}&
\frac{1}{(m-1)!}
\int_0^1(1-\tau)^{m-1}\times\mathcal D^mf_0(X_0+\tau h_{D})[h_{D},\ldots,h_{D}]\,d\tau.
\end{align*}
Write
\(
\Delta_{D}(X)\coloneq f_0(X)-f^{(D)}_0(X),
\)
then the mean value theorem gives
\begin{equation}
\label{eq:projection-error}
|\Delta_{D}(X)|
\le
C_0\sum_{d>D}w_d|\xi_d|
\le
C_0H(D).
\end{equation}
We now construct a network approximation of the Taylor polynomial in
\eqref{taylor}. Given \(L>0\) and a positive sequence
\(\{\ell_d\}\), define
\(
D_0\coloneq\max\{d\colon L\ell_ds_d\ge1\},
\
L_d\coloneq\lfloor L\ell_ds_d\rfloor,
\ d=1,\ldots,D_0.
\)
For \(d\le D_0\), partition \([-s_d,s_d]\) into \(L_d\) intervals and
choose \(\delta_d\in(0,2s_d/(3L_d)]\). 
Denote the resulting cells in
\(\Omega'_{D_0}\coloneq\Omega_{D_0}\setminus
\mathcal R(\Omega_{D_0},\bm L,\bm\delta)\)
by \(P_{\bm\theta}^{(D_0)}\), where
\(\theta_d\in\{0,\ldots,L_d-1\}\), and define
\begin{align*}
P_{\bm\theta}^{(D)}
&\coloneq P_{\bm\theta}^{(D_0)}
\times\prod_{d=D_0+1}^{D}[-s_d,s_d],\\
X_{\bm\theta}(t)
&\coloneq
\sum_{d=1}^{D_0}
\left(-s_d+\frac{2s_d\theta_d}{L_d}\right)\nu_d(t).
\end{align*}
For \(\bm x^{(D)}\in P_{\bm\theta}^{(D)}\), define
\[
\xi_{h_{\bm\theta},d}
\coloneq
\begin{cases}
\displaystyle
\xi_d+s_d-\frac{2s_d\theta_d}{L_d},
&1\le d\le D_0,\\[6pt]
\xi_d,
&D_0<d\le D.
\end{cases}
\]
The remainder of the Taylor expansion satisfies
\begin{align}
\label{eq:taylor-remainder-expansion}
|R_{m,D}(X,X_{\bm\theta})|
&\le\frac{C_0}{m!}
\sum_{d_1,\ldots,d_m=1}^{D}
\prod_{j=1}^m w_{d_j}
|\xi_{h_{\bm\theta},d_j}|
\nonumber\\
&=\frac{C_0}{m!}
\left(\sum_{d=1}^{D}w_d|\xi_{h_{\bm\theta},d}|\right)^m
\nonumber\\
&=\frac{C_0}{m!}\left\{
\sum_{d=1}^{D_0}w_d
\left|\xi_d+s_d-\frac{2s_d\theta_d}{L_d}\right|
+\sum_{d=D_0+1}^{D}w_d|\xi_d|
\right\}^m.
\end{align}
For \(\bm x^{(D_0)}\in P_{\bm\theta}^{(D_0)}\),
\[
\left|
\xi_d+s_d-\frac{2s_d\theta_d}{L_d}
\right|
\le
\frac{2s_d}{L_d},
\qquad d=1,\ldots,D_0.
\]
Since \(L_d=\lfloor L\ell_ds_d\rfloor\) and
\(L\ell_ds_d\geq1\) for \(d\leq D_0\), we have
\(2s_d/L_d\lesssim L^{-1}\ell_d^{-1}\). Consequently, for
\(\bm x^{(D)}\in P_{\bm\theta}^{(D)}\),
\(
|\xi_{h_{\bm\theta},d}|
\lesssim
L^{-1}\ell_d^{-1},
\ d\leq D_0.
\)
In addition, since \(|\xi_d|\le s_d\), it follows that
\begin{equation}
\label{s3}
\begin{aligned}
|R_{m,D}(X,X_{\bm\theta})|
&\lesssim
\left\{
L^{-1}\sum_{d=1}^{D_0}w_d\ell_d^{-1}
+\sum_{d=D_0+1}^{D}w_ds_d
\right\}^m\\
&=
\left\{
L^{-1}\sum_{d=1}^{D_0}w_d\ell_d^{-1}
+H(D_0)
\right\}^m.
\end{aligned}
\end{equation}

For \(s=1,\ldots,m-1\) and
\(\bm d=(d_1,\ldots,d_s)\in\{1,\ldots,D\}^s\), define the
piecewise-constant Taylor coefficient map
\begin{align*}
\mathcal Q_{\bm d}^{(s)}(\bm x^{(D)})
&\coloneq
\sum_{\bm\theta}
\mathcal D^sf_0(X_{\bm\theta})
[\nu_{d_1},\ldots,\nu_{d_s}]
\mathbf 1_{P_{\bm\theta}^{(D_0)}}(\bm x^{(D_0)}).
\end{align*}
Also define
\[
\mathcal Q^{(0)}(\bm x^{(D)})
\coloneq
\sum_{\bm\theta}
f_0(X_{\bm\theta})
\mathbf 1_{P_{\bm\theta}^{(D_0)}}(\bm x^{(D_0)}).
\]
For each fixed \(\bm d\), these maps have \(\prod_{d=1}^{D_0}L_d\) distinct values. 
The Taylor polynomial to be approximated is
\begin{align*}
A_{\bm\theta}^{(D)}(\bm x^{(D)})
\coloneq{}&f_0(X_{\bm\theta})
+\sum_{s=1}^{m-1}\frac{1}{s!}
\sum_{d_1,\ldots,d_s=1}^{D}
\mathcal D^sf_0(X_{\bm\theta})[\nu_{d_1},\ldots,\nu_{d_s}]\times
\prod_{j=1}^s\xi_{h_{\bm\theta},d_j}.
\end{align*}
The network construction steps are as follows:

\noindent
\textbf{Step 1 (Discretization).}
By Lemma~\ref{lem:B1}, for \(d=1,\ldots,D_0\), there exists a ReLU
network \(\psi_d\) with width \(6W_{\psi,d}\) and depth \(3M_\psi\)
such that \(W_{\psi,d}^2M_\psi^2=L_d\) and
\[
\psi_d(\xi_d)
=
-s_d+\frac{2s_di}{L_d}
\]
whenever
\begin{align*}
\xi_d\in\Bigl[&-s_d+\frac{2s_di}{L_d},-s_d+\frac{2s_d(i+1)}{L_d}
-\delta_d\mathbf 1_{\{i<L_d-1\}}\Bigr],\ i=0,\ldots,L_d-1.
\end{align*}
For \(\bm x^{(D_0)}\in P_{\bm\theta}^{(D_0)}\),
\[
\tilde\psi(\bm x^{(D_0)})
=\left(-s_d+\frac{2s_d\theta_d}{L_d}\right)_{d=1}^{D_0}.
\]
The network \(\tilde\psi(\cdot)\) has width
\(6\sum_{d=1}^{D_0}W_{\psi,d}\) and depth \(3M_\psi\).
The index set
\(\{0,\ldots,L_1-1\}\times\cdots\times\{0,\ldots,L_{D_0}-1\}\)
is in one-to-one correspondence with
\(\{0,\ldots,\prod_{d=1}^{D_0}L_d-1\}\) through
\(i_{\bm\theta}\coloneq\sum_{d=1}^{D_0}\theta_d
\prod_{k=1}^{d-1}L_k\), where
\(\theta_d\in\{0,1,\ldots,L_d-1\}\).
Accordingly, define
\[\psi_{\mathrm{enc}}(\bm x^{(D_0)})
\coloneq\sum_{d=1}^{D_0}\frac{L_d}{2s_d}
\{\psi_d(\xi_d)+s_d\}\prod_{k=1}^{d-1}L_k.\]
For every \(\bm x^{(D_0)}\in P_{\bm\theta}^{(D_0)}\), we have \(\psi_{\mathrm{enc}}(\bm x^{(D_0)})=i_{\bm\theta}\).

\noindent
\textbf{Step 2 (Approximation of the Taylor coefficients).}
Let \(c_{s,\bm d}\coloneq C_0\prod_{j=1}^sw_{d_j}\) and
\(e_{s,\bm d}\coloneq c_{s,\bm d}\prod_{j=1}^ss_{d_j}
=C_0\prod_{j=1}^sw_{d_j}s_{d_j}\).
Assumption \ref{ass2} gives
\(
\left|\mathcal D^sf_0(X_{\bm\theta})[\nu_{d_1},\ldots,\nu_{d_s}]\right|\le c_{s,\bm d}.
\)
The number of cells satisfies
\(\prod_{d=1}^{D_0}L_d=L^{D_0}\prod_{d=1}^{D_0}\ell_ds_d\).
Define \(\Psi(D_0)\coloneq-\sum_{d=1}^{D_0}\log(\ell_ds_d)\),
so that we may equivalently write
\(
\prod_{d=1}^{D_0}L_d= L^{D_0}e^{-\Psi(D_0)}.
\)
For each
\(\bm d=(d_1,\ldots,d_s)\in\{1,\ldots,D\}^s\), Lemma 2, composed with the radix encoder
$\psi_{\mathrm{enc}}$, gives a ReLU network
$\Phi_{\bm d,s}$ satisfying 
\begin{equation}
\label{partial}
\left|
\Phi_{\bm d,s}(\bm x^{(D)})
-
\mathcal Q_{\bm d}^{(s)}(\bm x^{(D)})
\right|
\lesssim
c_{s,\bm d}
\left\{
L^{D_0}e^{-\Psi(D_0)}
\right\}^{-m}.
\end{equation}
Similarly, there exists a network \(\Phi_0\) satisfying
\begin{equation}
\label{partial0}
\left|
\Phi_0(\bm x^{(D)})
-
\mathcal Q^{(0)}(\bm x^{(D)})
\right|
\lesssim
\left\{
L^{D_0}e^{-\Psi(D_0)}
\right\}^{-m}.
\end{equation}
Each \(\Phi_{\bm d,s}\) has width
\(
16rW_\phi\lceil\log_2(8W_\phi)\rceil
\)
and depth
\(
5(M_\phi+2)\lceil\log_2(4M_\phi)\rceil+3M_\psi,
\)
where
\(
W_\phi^2M_\phi^2=L^{D_0}e^{-\Psi(D_0)}.
\)
Since \(\bm d\in\{1,\ldots,D\}^s\), the parallel network for all
order-\(s\) coefficient maps has width
\(
16rD^sW_\phi\lceil\log_2(8W_\phi)\rceil.
\)
Notice that \(D^s\) counts the number of coefficient maps, whereas
each coefficient map is fitted on only
\(\prod_{d=1}^{D_0}L_d\) locations.

\medskip
\noindent
\textbf{Step 3 (Monomial approximation).}
For \(\bm d=(d_1,\ldots,d_s)\in\{1,\ldots,D\}^s\), Lemma~\ref{lem:B4}, applied on the corresponding coordinate ranges, gives a ReLU network \(J_{\bm d,s}\) with width \(9W_J+s-1\) and depth \(7s^2M_J\) such that
\begin{equation}
\label{prod-monomial}
\left|J_{\bm d,s}(\bm x^{(D)})
-\prod_{j=1}^s\xi_{h_{\bm\theta},d_j}\right|\le
2^s\left(\prod_{j=1}^ss_{d_j}\right)9sW_J^{-7sM_J}.
\end{equation}
The coordinate-dependent factor follows by applying the multiplication
network on the corresponding bounded intervals; it does not require a
rescaling of the input domain.
By Lemma~\ref{lem:B3}, the multiplication network
\(\Phi_\times\) can be chosen so that, on the relevant bounded
rectangles,
\begin{equation}
\label{prod}
\left|
\Phi_\times(u,v)-uv
\right|
\lesssim e_{s,\bm d}W_\times^{-M_\times}.
\end{equation}

Combining the preceding networks, define
\begin{align*}
\widehat A_{\bm\theta}^{(D)}(\bm x^{(D)})
=\Phi_0(\bm x^{(D)})+\sum_{s=1}^{m-1}\sum_{d_1,\ldots,d_s=1}^{D}
\Phi_\times\left(
\frac{\Phi_{\bm d,s}(\bm x^{(D)})}{s!},
J_{\bm d,s}(\bm x^{(D)})\right).
\end{align*}
Using \eqref{partial}--\eqref{prod} and the triangle inequality gives
\begin{align*}
\left|\widehat A_{\bm\theta}^{(D)}(\bm x^{(D)})
-A_{\bm\theta}^{(D)}(\bm x^{(D)})\right|\lesssim \sum_{s=0}^{m-1}\frac{1}{s!}
\sum_{d_1,\ldots,d_s=1}^{D}e_{s,\bm d}
\big\{(W_\phi^2M_\phi^2)^{-m}
+9sW_J^{-7sM_J}+6W_\times^{-M_\times}\big\}.
\end{align*}
Indeed,
\begin{align*}
\sum_{d_1,\ldots,d_s=1}^{D}
\prod_{j=1}^sw_{d_j}s_{d_j}
&=\left(\sum_{d=1}^{D}w_ds_d\right)^s\le\left(\sum_{d=1}^{\infty}w_ds_d\right)^s<\infty.
\end{align*}
Since \(\sum_{d=1}^{\infty}w_ds_d<\infty,\) we obtain
\begin{equation}
\label{eq:network-approximation}
\begin{aligned}
\left|\widehat A_{\bm\theta}^{(D)}(\bm x^{(D)})
-A_{\bm\theta}^{(D)}(\bm x^{(D)})\right|\lesssim (W_\phi^2M_\phi^2)^{-m}
+\max_{1\le s\le m-1}9sW_J^{-7sM_J}+6W_\times^{-M_\times}.
\end{aligned}
\end{equation}
To attain an approximation error \(\varepsilon\), the coefficient-fitting term requires \(W_\phi M_\phi\gtrsim\varepsilon^{-1/(2m)},\) whereas the monomial and multiplication networks require only
\(
M_J\log W_J\gtrsim\log(1/\varepsilon),\ M_\times\log W_\times\gtrsim\log(1/\varepsilon).
\)
Therefore, the coefficient-fitting networks determine the leading
complexity. Since the number of coefficient maps is bounded by
\(\sum_{s=0}^{m-1}D^s=\mathcal O(D^{m-1})\), the combined network can be chosen, up to the logarithmic factors appearing above, with
\begin{align*}
W=\mathcal O(D^mW_\phi),\ M=\mathcal O(M_\phi),\ SM=\mathcal O(D^mW_\phi^2M_\phi^2).
\end{align*}
In particular, with \(\mathcal A^2\coloneq W_\phi^2M_\phi^2,\) the total network complexity is \(SM=\mathcal O(D^mA^2).\) Thus 
\[
\left|\widehat A_{\bm\theta}^{(D)}(\bm x^{(D)})
-A_{\bm\theta}^{(D)}(\bm x^{(D)})\right|\lesssim (W_\phi^2M_\phi^2)^{-m}.
\]
Consequently, combining \eqref{s3} and
\eqref{eq:network-approximation} gives, uniformly over
\(\bm x^{(D)}\in\Omega_D\setminus\mathcal R_D\),
\begin{equation}
\label{eq:total-approximation}
\begin{aligned}
\left|\widehat A_{\bm\theta}^{(D)}(\bm x^{(D)})
-f_0^{(D)}(X)\right|\lesssim
\left\{L^{-1}\sum_{d=1}^{D_0}w_d\ell_d^{-1}
+H(D_0)\right\}^m
\end{aligned}
\end{equation}
After balancing
\(
L^{-1}\sum_{d=1}^{D_0}w_d\ell_d^{-1}=H(D_0)
\)
the Taylor remainder is bounded by a constant multiple of
\(H(D_0)^m\), 
so the network approximation error is of smaller order. Therefore,
the leading scale in \eqref{eq:total-approximation} is
\(H(D_0)^m\), determined by the Taylor remainder. 

We now optimize the grid allocation.  We balance the retained local
discretization error with \(H(D_0)\) by choosing
\begin{equation}
\label{opt-sol}
L=\frac{\sum_{d=1}^{D_0}w_d\ell_d^{-1}}{H(D_0)}.
\end{equation}
The coefficient-fitting complexity for each coefficient map satisfies
\(
\mathcal A^2=W_\phi^2M_\phi^2=L^{D_0}\prod_{d=1}^{D_0}\ell_ds_d.
\)
Substituting \eqref{opt-sol} gives
\begin{equation}
\label{budget}
\mathcal  A^2=H(D_0)^{-D_0}
\left(\sum_{d=1}^{D_0}w_d\ell_d^{-1}\right)^{D_0}\times\prod_{d=1}^{D_0}\ell_ds_d.
\end{equation}

For fixed \(D_0\), minimizing the required point-fitting budget \eqref{budget} is
equivalent to minimizing
\[
\left(
\sum_{d=1}^{D_0}w_d\ell_d^{-1}
\right)^{D_0}
\prod_{d=1}^{D_0}\ell_d
\]
over $\{\ell_d\colon d\ge1\}$. Let \(a_d\coloneq w_d\ell_d^{-1}\). Then
\[
\left(
\sum_{d=1}^{D_0}w_d\ell_d^{-1}
\right)^{D_0}
\prod_{d=1}^{D_0}\ell_d
=
\frac{
\left(\sum_{d=1}^{D_0}a_d\right)^{D_0}
}{
\prod_{d=1}^{D_0}a_d
}
\prod_{d=1}^{D_0}w_d.
\]
By the arithmetic--geometric mean inequality,
\[
\frac{
\left(\sum_{d=1}^{D_0}a_d\right)^{D_0}
}{
\prod_{d=1}^{D_0}a_d
}
\ge D_0^{D_0},
\]
with equality if and only if \(a_1=\cdots=a_{D_0}\).
Therefore, up to a common scaling absorbed into \(L\), the optimal
allocation is
\(
\ell_d=w_d.
\)
Consequently,
\[
\mathcal  A^2=L^{D_0}\prod_{d=1}^{D_0}w_ds_d
=L^{D_0}e^{-\Psi(D_0)},
\]
The balancing relation \eqref{opt-sol} becomes \(L^{-1}D_0=H(D_0).\) since \(\mathcal A^2=\prod_{d=1}^{D_0}L_d\ge L_1= L\) and the balancing relation gives \(H(D_0)=D_0/L\), we have \(\mathcal A^{-2m}\le L^{-m}=D_0^{-m}H(D_0)^m\le H(D_0)^m.\) Therefore, the network approximation error is no larger than the Taylor remainder.

It follows that
\[
\mathcal A^2
=
\left(
\frac{D_0}{H(D_0)}
\right)^{D_0}
e^{-\Psi(D_0)},
\]
and hence \(D_0\) is determined by
\[
\log \mathcal A^2
=
D_0\log\left(\frac{D_0}{H(D_0)}\right)
-\Psi(D_0).
\]
Thus,
\[
L
=
\exp\left[
\frac{1}{D_0}
\{\log \mathcal A^2+\Psi(D_0)\}
\right].
\]
Substituting these bounds into \eqref{eq:total-approximation} yields
\begin{align*}
\|\widetilde f-f_0^{(D)}\|_{L^\infty(
\Omega_D\setminus\mathcal R_D)}&\lesssim H(D_0)^m+\mathcal  A^{-2m}\lesssim H(D_0)^m\\
&=D_0^m
\exp\left[
-\frac{m}{D_0}
\{\log \mathcal  A^2+\Psi(D_0)\}
\right].
\end{align*}
For each fixed \(s\) and \(\bm d=(d_1,\ldots,d_s)\), the
coefficient-fitting network \(\Phi_{\bm d,s}\) can be chosen, up to
logarithmic factors, with size \(\mathcal O(W_\phi^2M_\phi)\), width
\(\mathcal O(W_\phi)\), and depth \(\mathcal O(M_\phi)\). The number
of coefficient maps appearing in the Taylor polynomial is
\[
\sum_{s=0}^{m-1}D^s
\le
mD^{m-1}
=
\mathcal O(D^{m-1}).
\]
Since these coefficient-fitting subnetworks are
constructed in parallel, their sizes and widths are additive, whereas
their depths are given by the maximum depth. Therefore, the combined
coefficient-fitting network satisfies
\begin{align*}
S_{\mathrm{coef}}\lesssim D^{m-1}W_\phi^2M_\phi,\ W_{\mathrm{coef}}\lesssim D^{m-1}W_\phi,\ M_{\mathrm{coef}}\lesssim M_\phi.
\end{align*}
The discretization network is shared by all coefficient maps, while the
monomial and multiplication subnetworks require only logarithmic
complexity in the target accuracy. Compared with the coefficient-fitting subnetworks, the complexities of the shared discretization network and the monomial and multiplication subnetworks are negligible and therefore do not affect the leading orders of the network size, width, and depth. Before removing the trifling region, the resulting network therefore satisfies
\begin{align*}
\widetilde S\lesssim D^{m-1}W_\phi^2M_\phi,\ \widetilde W\lesssim D^{m-1}W_\phi,\ \widetilde M\lesssim M_\phi.
\end{align*}

It remains to extend the preceding approximation from
\(\Omega_D\setminus\mathcal R_D\) to all of \(\Omega_D\). Because the
construction is valid for every positive
\(\delta_d\leq2s_d/(3L_d)\), choose
\[
\delta_d
\leq
\min\left\{
\frac{2s_d}{3L_d},
\frac{H(D_0)^m}{C_0D_0w_d}
\right\},
\qquad d=1,\ldots,D_0.
\]
Then \(C_0\sum_{d=1}^{D_0}w_d\delta_d\leq H(D_0)^m\).
Applying Lemma~\ref{lem:remove-trifling-region} with
\(g=f_0^{(D)}\) gives a ReLU network \(f^\sharp\) satisfying
\begin{align*}
\|f^\sharp-f_0^{(D)}\|_{L^\infty(\Omega_D)}
&\lesssim
H(D_0)^m
+C_0\sum_{d=1}^{D_0}w_d\delta_d\lesssim
H(D_0)^m.
\end{align*}
The projection error satisfies \(H(D)\asymp H(D_0)^m\).
The same lemma gives
\(
S\lesssim3^{D_0}\{D^{m-1}W_\phi^2M_\phi+1\},
\)
\(
W\lesssim3^{D_0}\{D^{m-1}W_\phi+1\},
\)
and \(M\lesssim M_\phi+2D_0\). This proves the asserted global
uniform approximation bound.
\end{proof}

\begin{proof}[Proof of Theorem \ref{lower}]

For each \(D\ge1\), consider the finite-dimensional slice
\(
\Omega_{D}=\prod_{d=1}^D[-s_d,s_d],
\)
embedded into \(\Omega\) by setting all coordinates \(d>D\) equal to zero. 
Any family constructed on this slice can be extended to the full sequence domain by ignoring the tail coordinates. 
Hence the finite-dimensional packing argument can be carried out on \(\Omega_{D}\).
We partition the \(d\)-th coordinate interval into cells of length of order \(L^{-1}\ell_d^{-1}\), where \(L>0\) is a global resolution parameter and \(\ell_d\) is a coordinate-specific grid allocation factor, with \(\ell_1=1\). 
Thus the number of cells in the \(d\)-th coordinate is \(N_d=\lfloor Ls_d\ell_d\rfloor\), provided \(Ls_D\ell_D\ge2\). 
Let \(x_{\bm\beta}\) denote the center of a grid cell, where \(\bm\beta=(\beta_1,\ldots,\beta_D)\) and \(\beta_d\in\{0,\ldots,N_d-1\}\). 
Choose a fixed bump function \(\varphi\in C^m(\mathbb R^D)\) such that \(\varphi(0)=1\), \(\varphi\) is supported in a sufficiently small ball, and
\[
\max_{\|\alpha\|_1\le m}\|\partial^\alpha\varphi\|_\infty\le1.
\]
Define the anisotropic bump
\(
\varphi_{\bm\beta}(x)
=
B_L
\varphi\left(
L\,\operatorname{diag}(\ell_1,\ldots,\ell_D)(x-x_{\bm\beta})
\right),
\)
where $\psi(0)=1$.
For any multi-index \(\alpha\), the chain rule gives
\[
|\partial^\alpha \varphi_{\bm\beta}(x)|
\le
B_L L^{\|\alpha\|_1}
\prod_{d=1}^D\ell_d^{\alpha_d}.
\]
To ensure \(\varphi_{\bm\beta}\in \mathcal C^m(\Omega_{D},r)\), it is enough that
\[
B_L L^{\|\alpha\|_1}
\prod_{d=1}^D
\ell_d^{\alpha_d}
\le
C_0\prod_{d=1}^D w_d^{\alpha_d}.
\]
Equivalently,
\[
B_L\le C_0L^{-\|\alpha\|_1}
\prod_{d=1}^D
(w_d\ell_d^{-1})^{\alpha_d}.
\]
We now determine the optimal choice of the grid allocation \(\ell_d\). 
The derivative constraint requires
\[
B_L\le C_0\min_{1\le \|\alpha\|_1\le m}
L^{-\|\alpha\|_1}
\prod_{d=1}^D
(w_d\ell_d^{-1})^{\alpha_d}.
\]
At the same time, the grid cardinality satisfies
\[
\prod_{d=1}^D N_d
=
L^D\prod_{d=1}^Ds_d\ell_d,
\]
so larger values of \(\ell_d\) lead to a larger shattering set. 
Hence the optimal construction takes \(\ell_d\) as large as permitted by the derivative constraint.

We consider two regimes.

\emph{Case 1: \(\ell_d/w_d\to0\).}
In this case, \(w_d\ell_d^{-1}\) is increasing across \(d\). 
The minimum in the derivative constraint is attained by placing all \(m\) derivatives on the first coordinate. 
Hence \(B_L= C_0L^{-m}\). 
Taking the critical scaling \(B_L=2\varepsilon\) gives \(L=C^{1/m}_0(2\varepsilon)^{-1/m}\). 
The grid cardinality is
\[
\prod_{d=1}^D N_d
=
L^D\prod_{d=1}^Ds_d\ell_d.
\]
The largest possible entropy is achieved at the boundary value \(\ell_d=w_d\).

\emph{Case 2: \(\ell_d/w_d\to\infty\).}
In this case, \(w_d\ell_d^{-1}\) is decreasing across $d$. 
The derivative constraint is now most restrictive at the largest coordinate \(D\). Hence \(B_L= C_0L^{-m}(w_d\ell_D^{-1})^m.\)
Taking \(B_L=2\varepsilon\) gives \(L=C_0^{1/m}(2\varepsilon)^{-1/m}w_d\ell_D^{-1}.\)
Substituting this into the grid cardinality,
\begin{align*}
\prod_{d=1}^D N_d
&=L^D\prod_{d=1}^Ds_d\ell_d=C_0^{D/m}(2\varepsilon)^{-D/m}
w_d^D\ell_D^{-D}\times\prod_{d=1}^Ds_d\ell_d.
\end{align*}
Now fix all coordinates except \(\ell_D\). 
Then the above expression becomes
\[
C_0^{D/m}(2\varepsilon)^{-D/m}
w_d^D
\left(
\prod_{d=1}^{D-1}s_d\ell_d
\right)
s_D\ell_D^{-(D-1)}.
\]
Since \(D\ge2\), the factor \(\ell_D^{-(D-1)}\) is decreasing in \(\ell_D\). 
Therefore the grid cardinality is maximized by taking the smallest admissible value of \(\ell_D\). 
On the other hand, the present regime assumes \(w_d\ell_D^{-1}\le1\), whose boundary is \(\ell_D=w_d\). 
Hence the smallest admissible choice is exactly the boundary value \(\ell_D=w_d\). 
Repeating the same argument coordinatewise yields \(\ell_d=w_d,\ d=1,\ldots,D.\)
Combining the two cases, the optimal grid allocation is \(\ell_d=w_d\), and the derivative constraint reduces to
\(B_L\le C_0L^{-m}.\)
Under the envelope representation \(w_ds_d=\exp\{-\psi(d)\}\), the grid cardinality becomes
\begin{align*}
\prod_{d=1}^D N_d
=L^D\prod_{d=1}^D w_ds_d=L^D\exp\{-\Psi(D)\},\ \Psi(D)=\sum_{d=1}^D\psi(d).
\end{align*}
Taking \(B_L=2\varepsilon\) and maximizing grid cardinality give \(L=C^{1/m}_0(2\varepsilon)^{-1/m}\). 

Define
\[
\mathcal B\coloneq\left\{\chi\colon\prod_{d=1}^D\{0,1,\ldots,L_d-1\}\to\{-1,1\}\right\}.
\]
For each sign pattern \(\chi\in\mathcal B\), define
\[
f_\chi(\bm x)\coloneq\sum_{\bm\beta}\chi(\bm\beta)\psi_{\bm\beta}(\bm x),\qquad \bm x\in\Omega_D.
\]
Since the supports of the bumps are disjoint, each \(f_\chi\) satisfies the anisotropic derivative bounds on \(\Omega_D\). 
By the uniform approximation property on the infinite-dimensional class, for every \(\chi\in\mathcal B\), there exists \(\phi_\chi\in\mathcal F\) such that
\(
\|\phi_\chi-F_\chi\|_{L^\infty(\Omega)}\le\varepsilon.
\)
At the embedded center \(\iota_D(\bm x_{\bm\beta})\), we have
\(
|F_\chi(\iota_D(\bm x_{\bm\beta}))|
=
|f_\chi(\bm x_{\bm\beta})|
=
B_L.
\)
If \(B_L=2\varepsilon\), then
\[
|\phi_\chi(\iota_D(\bm x_{\bm\beta}))-F_\chi(\iota_D(\bm x_{\bm\beta}))|
\le\varepsilon
<
|F_\chi(\iota_D(\bm x_{\bm\beta}))|.
\]
Therefore,
\(
\operatorname{sign}\{\phi_\chi(\iota_D(\bm x_{\bm\beta}))\}
=
\operatorname{sign}\{F_\chi(\iota_D(\bm x_{\bm\beta}))\}
=
\chi(\bm\beta).
\)
Thus \(\{\phi_\chi\colon\chi\in\mathcal B\}\subseteq\mathcal F\) shatters the embedded grid centers
\[
\left\{
x_{\bm\beta}\colon
\bm\beta\in
\prod_{d=1}^D\{0,1,\ldots,N_d-1\}
\right\}.
\]
Consequently,
\begin{align*}
V=\operatorname{PDim}(\mathcal F)
&\ge\prod_{d=1}^D N_d=C_0^{D/m}(2\varepsilon)^{-D/m}
\exp\{-\Psi(D)\}.
\end{align*}
Thus, for every admissible \(D\) satisfying \(\left(\frac{C_0}{2\varepsilon}\right)^{1/m}w_Ds_D\geq1\),
\[
\varepsilon
\ge
\frac{C_0}{2}
\exp\left[
-\frac{m}{D}\{\log V+\Psi(D)\}
\right].
\]
Combining the fact that 
\begin{align*}
&\sup_{f\in\mathcal C^m(\Omega,\bm w)}\inf_{\phi\in\mathcal F}\|\phi-f\|_{L^\infty(\Omega)}\ge \sup_{f\in\mathcal C^m(\Omega_D,\bm w)}\inf_{\phi\in\mathcal F}\|\phi-f\|_{L^\infty(\Omega_D)}
\end{align*}
completes the proof.
\end{proof}

\begin{proof}[Proof of Theorem \ref{thm:exp-approx-rate}]
Put \(a_d=w_ds_d\) and \(B(d)=d\log(2/a_d)-\Psi(d)\), then \(D_0=\max\{d\colon B(d)\leq t\}\). By the assumed uniform decay equivalence, there is a constant \(C_e\), independent of \(d\), \(\rho\), and \(t\), such that
\begin{equation}
\label{eq:exp-uniform-log-envelope}
\left|\log(a_d^{-1})-c_ad^\rho\right|\leq C_e.
\end{equation}
For every \(d\geq1\) and every \(\rho>0\), monotonicity of
\(x^\rho\) and integral comparison give
\begin{equation*}
\label{eq:exp-uniform-power-sum}
\frac{d^{\rho+1}}{\rho+1}
\leq\sum_{k=1}^d k^\rho
\leq\frac{d^{\rho+1}}{\rho+1}+d^\rho.
\end{equation*}
Consequently,
\begin{equation}
\label{eq:exp-uniform-psi-expansion}
\Psi(d)=\frac{c_a}{\rho+1}d^{\rho+1}+R_\Psi(d),\ |R_\Psi(d)|\leq C(c_ad^\rho+d),
\end{equation}
where \(C\) is independent of \(d\), \(t\), and \(\mathcal A\).
Substituting \eqref{eq:exp-uniform-log-envelope} and
\eqref{eq:exp-uniform-psi-expansion} into the definition of \(B(d)\)
gives
\begin{align}
\label{eq:exp-uniform-threshold-expansion}
B(d)=\frac{\rho c_a}{\rho+1}d^{\rho+1}+R_B(d),\ |R_B(d)|\leq C(c_ad^\rho+d).
\end{align}
By the definition of \(D_0\) in Theorem~\ref{upper},
\begin{equation}
\label{eq:exp-uniform-threshold-bracket}
B(D_0)\leq t<B(D_0+1).
\end{equation}
Combining \eqref{eq:exp-uniform-threshold-expansion} and
\eqref{eq:exp-uniform-threshold-bracket} yields the uniform
budget--dimension bracket
\begin{align}
 t&\geq
\frac{\rho c_a}{\rho+1}D_0^{\rho+1}
-C(c_aD_0^\rho+D_0)
\nonumber\\
t&<\frac{\rho c_a}{\rho+1}(D_0+1)^{\rho+1}+C\{c_a(D_0+1)^\rho+D_0+1\}.
\label{eq:exp-uniform-budget-bracket}
\end{align}
Set
\(G_{t,\rho}=D_0^{-1}\{t+\Psi(D_0)\}-\log D_0\). By
Theorem~\ref{upper},
\begin{align}
\inf_{f\in\mathcal F}\|f-f_0\|_{L^\infty}
&\lesssim \exp(-mG_{t,\rho})\nonumber\\
&\lesssim \exp[-m\{G_{t,\rho}-\Psi(1)\}].
\label{eq:exp-unified-effective-exponent}
\end{align}
where the last implicit constant is independent of \(t\) and \(\rho\).
By the definition of \(G_{t,\rho}\), the numerator in
\eqref{remainder} is \(G_{t,\rho}-\Psi(1)\). Hence
\begin{align*}
G_{t,\rho}-\Psi(1)
={}&
c_a^{1/(\rho+1)}
\left\{\frac{t(\rho+1)}{\rho}\right\}^{\rho/(\rho+1)}
(1+\delta_{\mathcal A,\rho}).
\end{align*}
Multiplying both sides by \(m\) and recalling that
\begin{align*}
c'
=
mc_a^{1/(\rho+1)}
\left(\frac{\rho+1}{\rho}\right)^{\rho/(\rho+1)},
\end{align*}
we obtain
$m\{G_{t,\rho}-\Psi(1)\}=c't^{\rho/(\rho+1)}
(1+\delta_{\mathcal A,\rho})$.
Therefore,
\begin{align*}
\exp[-m\{G_{t,\rho}-\Psi(1)\}]
=
\exp\{-c't^{\rho/(\rho+1)}
(1+\delta_{\mathcal A,\rho})\}.
\end{align*}
Substituting this identity into
\eqref{eq:exp-unified-effective-exponent} yields
\eqref{eq:exp-uniform-explicit-rate}.

 It remains only to verify the two asserted limits of $\delta_{\mathcal A}$. First fix \(\rho>0\) and let \(t\to\infty\). Let
\(x_{\rho,t}=\{(\rho+1)t/(\rho c_a)\}^{1/(\rho+1)}\) and
\(F_\rho(x)=\rho c_ax^{\rho+1}/(\rho+1)\), so that
\(F_\rho(x_{\rho,t})=t\). Since \(\rho\) is fixed and
\(x_{\rho,t}\to\infty\),
\eqref{eq:exp-uniform-budget-bracket} and the mean-value theorem imply
\begin{equation}
\label{eq:exp-uniform-dimension-inversion}
D_0^\rho=x_{\rho,t}^\rho(1+\eta_{\rho,t}),
\
|\eta_{\rho,t}|\lesssim 
\frac{\rho+1}{x_{\rho,t}}
+\frac{1}{c_ax_{\rho,t}^\rho}.
\end{equation}
Taking
\eqref{eq:exp-uniform-psi-expansion},
\eqref{eq:exp-uniform-budget-bracket}, and
\eqref{eq:exp-uniform-dimension-inversion} together gives
\begin{align}
G_{t,\rho}-\Psi(1)
={}&c_ax_{\rho,t}^\rho
\Bigg\{1+O\Bigg(
\frac{\rho+1}{x_{\rho,t}}+\frac{1+|\log x_{\rho,t}|}{c_ax_{\rho,t}^\rho}
\Bigg)\Bigg\}.
\label{eq:exp-uniform-effective-exponent-expansion}
\end{align}
Since
\(mc_ax_{\rho,t}^\rho
=mc_a^{1/(\rho+1)}\{(\rho+1)/\rho\}^{\rho/(\rho+1)}
t^{\rho/(\rho+1)}\), the definition in \eqref{remainder} and
\eqref{eq:exp-uniform-effective-exponent-expansion} give
\begin{align*}
\delta_{\mathcal A,\rho}
=O\left\{
\frac{\rho+1}{x_{\rho,t}}
+\frac{1+|\log x_{\rho,t}|}{c_ax_{\rho,t}^\rho}
\right\}
=o(1)
\end{align*}
as \(t\to\infty\) for every fixed \(\rho>0\). 

We next fix \(t\geq\log2\) and let \(\rho\to\infty\). Write
\(e_d=\log(a_d^{-1})-c_ad^\rho\). By
\eqref{eq:exp-uniform-log-envelope}, \(|e_d|\leq C_e\), and hence
\begin{align}
B(d)={}&d\log2
+c_a\left\{d^{\rho+1}-\sum_{k=1}^d k^\rho\right\}
+de_d-\sum_{k=1}^d e_k.
\label{eq:large-rho-threshold}
\end{align}
For \(d\geq2\) and \(\rho\geq1\), the mean-value theorem gives
\begin{align*}
d^{\rho+1}-\sum_{k=1}^d k^\rho
&=\sum_{k=1}^{d-1}(d^\rho-k^\rho)\\
&\geq(d-1)\{d^\rho-(d-1)^\rho\}\\
&\geq\rho(d-1)^\rho
\geq\frac{\rho d}{2}.
\end{align*}
Since \(|de_d-\sum_{k=1}^d e_k|\leq2C_ed\),
\eqref{eq:large-rho-threshold} implies
\begin{equation}
B(d)\geq
d\left(\log2+\frac{c_a\rho}{2}-2C_e\right),
\qquad d\geq2.
\label{eq:large-rho-threshold-lower}
\end{equation}
For fixed \(t\), the right-hand side exceeds \(t\) for all \(d\geq2\)
once \(\rho\) is sufficiently large. On the other hand,
\(B(1)=\log2\leq t\). The definition of \(D_0\) therefore gives
\(D_0=1\). Consequently,
\(G_{t,\rho}-\Psi(1)=t\). At the same time,
\begin{equation}
\label{eq:exp-large-rho-continuous-exponent}
c_a^{1/(\rho+1)}
\left(\frac{\rho+1}{\rho}\right)^{\rho/(\rho+1)}
t^{\rho/(\rho+1)}\longrightarrow t.
\end{equation}
Therefore, \eqref{remainder} gives
\begin{align*}
1+\delta_{\mathcal A,\rho}
&=
\frac{t}
{c_a^{1/(\rho+1)}
\{(\rho+1)/\rho\}^{\rho/(\rho+1)}
t^{\rho/(\rho+1)}}\longrightarrow1.
\end{align*}
Thus \(\delta_{\mathcal A,\rho}\to0\) as \(\rho\to\infty\) with
\(t\) fixed. 
This completes the proof.
\end{proof}

\begin{lemma}[Proposition 4.3 in \citet{lu2021deep}]\label{lem:B1}
For any $W,M \in \mathbb{N}_+$ and $\delta_d \in (0,3K^{-1}]$ with $K = W^2M^2$, there exists a one-dimensional function $\phi$ implemented by a ReLU fully connected neural network with width $6W$ and depth $3M$ such that
\[
\phi(x) = j, \qquad
\text{if } x \in \Bigl[ \frac{j}{K},\frac{j+1}{K} - \delta \cdot \mathbbm{1}_{\{j < K-1\}} \Bigr],
\]
for $j = 0,1,\ldots,K-1$.
\end{lemma}
\begin{lemma}[Proposition 4.4 in \citet{lu2021deep}]\label{lem:B2}
Given any $W,M,s \in \mathbb{N}_+$ and $\xi_i \in [0,1]$ for
$i = 0,1,\ldots,W^{2}M^{2}-1$, there exists a function $\phi$ implemented by a ReLU fully
connected neural network with width $16s(W+1)\bigl\lceil \log_2(8W) \bigr\rceil$ and depth $5(M+2) \bigl\lceil \log_2(4M) \bigr\rceil$
such that
\[
\bigl| \phi(i) - \xi_i \bigr| \le W^{-2s} M^{-2s},
\qquad i = 0,1,\ldots,W^{2}M^{2}-1,
\]
and $0 \le \phi(x) \le 1,\ \forall\, x \in \mathbb{R}$.
\end{lemma}
\begin{lemma}[Lemma 4.2 in \citet{lu2021deep}]\label{lem:B3}
For any $W,M \in \mathbb{N}_+$ and $a,b \in \mathbb{R}$ with $a<b$, there exists a function
$\phi$ implemented by a ReLU fully connected neural network with width $9W+1$ and
depth $M$ such that $\bigl| \phi(x,y) - xy \bigr|
\le 6(b-a)^2W^{-M},\ \forall\, x,y \in [a,b]$.
\end{lemma}
\begin{lemma}[Proposition 4.1 in \citet{lu2021deep}]\label{lem:B4}
Assume $P_{{\bm{\alpha}}}(x) = x^{\bm{{\bm{\alpha}}}} = x_1^{{\bm{\alpha}}_1} x_2^{{\bm{\alpha}}_2} \cdots x_d^{\alpha_d}$ for
$\bm{{\bm{\alpha}}} \in \mathbb{N}_0^d$ with $\|\bm{{\bm{\alpha}}}\|_1 \le s$, where $s \in \mathbb{N}_+$. For any $W,M \in \mathbb{N}_+$, there exists a function $\phi$ implemented by a ReLU fully connected neural network with width $9W + s - 1$ and depth $7s^2 M$ such that
\[
\bigl| \phi(x) - P(x) \bigr|\le 9sW^{-7sM},\qquad \forall\, x \in [0,1]^d.
\]
\end{lemma}

\begin{lemma}[Anisotropic removal of the trifling region]
\label{lem:remove-trifling-region}
Let \(g\in C(\Omega_{D_0})\). Suppose that
\(0<\delta_d\leq2s_d/(3L_d)\) for
\(d=1,\ldots,D_0\), and let \(\widetilde f\) be a ReLU network of
size \(\widetilde S\), width \(\widetilde W\), and depth
\(\widetilde M\), defined on \(\Omega_{D_0}\). If
\(
|\widetilde f(\bm x^{({D_0})})-g(\bm x^{({D_0})})|\leq\varepsilon
\)
for every \(\bm x^{({D_0})}\in\Omega_D\setminus\mathcal R_{D_0}\), then there exists a ReLU network \(f^\sharp\) on \(\Omega_{D_0}\) such that
\[
\|f^\sharp-g\|_{L^\infty(\Omega_{D_0})}
\leq
\varepsilon+
\sum_{d=1}^{D_0}\omega_{g,d}(\delta_d).
\]
Moreover, \(f^\sharp\) may be chosen with width at most
\(3^{D_0}(\widetilde W+4)\), depth at most
\(\widetilde M+2D_0\), and size 
\(S\lesssim 3^{D_0}\{\widetilde S+1\}\).
If \(g=f_0^{({D_0})}\), then Assumption~\ref{ass2} gives
\(\omega_{f_0^{({D_0})},d}(\delta_d)\leq C_0w_d\delta_d\), and therefore
\[
\|f^\sharp-f_0^{({D_0})}\|_{L^\infty(\Omega_{D_0})}
\leq
\varepsilon+C_0\sum_{d=1}^{D_0}w_d\delta_d.
\]
\end{lemma}

\begin{proof}
The proof is the coordinatewise median-extension argument of
Theorem~2.1 in \citet{lu2021deep}. At step \(d\), use the three
versions of the current network obtained by leaving the \(d\)-th
coordinate unchanged and shifting it by \(\pm\delta_d\), with exact
ReLU clipping at the endpoints \(\pm s_d\), and take their median.
Because \(\delta_d\leq2s_d/(3L_d)\), at least two of these three
arguments avoid the \(d\)-th collection of boundary strips. The
median step therefore extends the approximation across those strips
and increases the error by at most \(\omega_{g,d}(\delta_d)\).
Repeating this operation for \(d=1,\ldots,D_0\) removes all coordinate
strips and gives the displayed error bound. The median of three real
numbers and the clipping maps are represented exactly by fixed-size
ReLU networks. Sequential application of the construction triples
the width and size at each coordinate and adds two layers, yielding
the stated complexity bounds. Finally, Assumption~\ref{ass2} and the
one-dimensional mean value theorem give
\(\omega_{f_0^{({D_0})},d}(\delta_d)\leq C_0w_d\delta_d\).
\end{proof}

\end{document}